%% file: ex_article.tex
\documentclass[hidelinks,onefignum,onetabnum]{siamart251216}

\input{ex_shared}

\ifpdf
\hypersetup{
  pdftitle={Coupled Optimal Transport with Landmark Constraints},
  pdfauthor={X. Gu, J. Sun, and Z. Xu}
}
\fi

\usepackage{amssymb}
\newcommand{\dif}{{\rm d}}

\begin{document}

\maketitle

\begin{abstract}
Existing optimal transport (OT) models primarily seek an OT map or plan between distributions by minimizing a prescribed transport cost or distortion. However, minimizing transport cost or distortion alone may fail to identify a geometrically meaningful transformation between the two distributions. To address this limitation, this paper proposes a novel coupled OT framework that leverages a small number of annotated landmarks to guide the recovery of an underlying deformation governing the distribution transformation. The coupled OT framework integrates the optimization of the transport plan and the deformation field into a unified model, where the landmark-guided deformation field and the cost-driven transport plan are coupled through a mutual-consistency constraint. As a result, the deformation is jointly determined by the annotated landmarks and cost-driven distribution matching. The proposed framework provides a principled connection between landmark-based registration and transport-based distribution matching, enabling the recovery of transport maps from sparse geometric supervision. We establish the well-definedness of the proposed model in a general variational setting and develop a finite-element-based numerical algorithm for computation whose convergence properties are systematically analyzed. The practical effectiveness of the proposed approach is verified in shape matching.

\end{abstract}



\section{Introduction}\label{sec:introduction}
Optimal Transport (OT), as a geometric tool for comparing or connecting distributions, has been successfully applied in diverse domains, such as astrophysics \cite{frisch2002reconstruction}, machine learning \cite{courty2016optimal,arjovsky2017wasserstein,lipman2023flow}, image science \cite{bauer2015diffeomorphic,solomon2015convolutional,gu2024adversarial}, and biology \cite{schiebinger2019optimal,klein2025mapping}.  OT aims to derive a transport map or plan between two probability distributions such that a certain transport cost is minimized. OT has two typical formulations, i.e., the Monge formulation \cite{monge1781memoire} and the Kantorovich formulation \cite{kantorovich1942translocation}. The Monge formulation optimizes among a set of continuous transport maps with a mass-preserving constraint, which is nonconvex and may be infeasible. The Kantorovich formulation relaxes the Monge formulation as a coupling or plan optimization problem and attracts broader studies in applications. The original Kantorovich problem is a linear program that is computationally expensive in its discrete version. The entropy-regularized OT \cite{cuturi2013sinkhorn} introduces the entropy of the transport plan as a regularization to the OT model, which is solved by the computationally cheaper Sinkhorn-Knopp algorithms \cite{cuturi2013sinkhorn,altschuler2017near,karlsson2017generalized}.
The original OT model has been extended in several aspects. To tackle the applications where only partial mass should be transported, partial OT \cite{guittet2002extended,figalli2010optimal,caffarelli2010free,chapel2020partial}, unbalanced OT \cite{liero2018optimal,chizat2018interpolating}, robust OT \cite{balaji2020robust,mukherjee2021outlier}, and supervised OT \cite{cang2022supervised} models have been proposed. On the other hand, to deal with the situations where the two measures are supported on different spaces and the transport cost is hard to define, Gromov-Wasserstein models \cite{sturm2006geometryI,memoli2011gromov} have been proposed to minimize the transport distortion defined on metrics within each space, respectively. 

The previous OT models mainly seek the transport map or plan under the criteria of transport cost or distortion minimization. The transport cost is often taken to be a distance, and the distortion is typically set as the difference between pairwise distances. However, such criteria may not faithfully characterize a geometrically meaningful distribution transformation. For example, in image or shape registration, two regions may be close but correspond to different semantic or anatomical parts. A transport map driven solely by distance minimization may therefore move mass along geometrically short but semantically incorrect paths, leading to mismatched correspondences and inaccurate descriptions of the true distribution deformation. A promising solution for this challenge is to leverage a small number of pre-annotated landmarks, which can typically be obtained at relatively low annotation cost. These landmarks provide sparse geometric information about the deformation between distributions and can guide a meaningful transport. Building on this idea, \cite{gu2022keypoint,gu2023keypoint} present a keypoint-guided OT model that introduces a mask-based constraint on the transport plan and a relation-preserving strategy to guide the matching of points, pioneering the study of landmark-driven OT. While inspiring, keypoint-guided OT targets the matching between discrete points and does not explicitly recover the deformation underlying the transition of distributions. This becomes restrictive in many applications such as image registration, shape analysis, and morphometric modeling, where one often seeks a deformation field corresponding to the distributional transition.

In this paper, we propose a novel coupled OT framework that utilizes sparse geometric landmarks to guide the recovery of the transformation between distributions. Our framework introduces a deformation-based view of transport: the transport plan should not only match the two distributions, but also be consistent with a coherent deformation suggested by the annotated landmarks. Under this spirit, we integrate the searching of the transport plan and the estimation of the deformation field into a unified coupled model. In this model, landmarks serve as local geometric anchors for the deformation field, which guides the transport behavior and is, in turn, constrained by the transport plan obtained through cost minimization. 
Consequently, the resulting transport is determined jointly by global cost minimization and a landmark-guided deformation that preserves local geometry.
Our framework provides a principled bridge between landmark-based registration and transport-based distribution matching, enabling the recovery of a geometrically meaningful transport guided by sparse landmark supervision. 
Following this model, extensive theoretical analysis is presented to establish its well-definedness and characterize its connection to conventional OT. We also devise a numerical algorithm based on the finite-element method to solve the model. The effectiveness of the proposed method is verified in the shape matching application.

\paragraph{Notations}
We use \(\mu_n\rightharpoonup\mu\) to denote narrow convergence for probability measures $\mu_n$ and $\mu$.  \(u_n\rightharpoonup u\) denotes weak convergence for functions $u_n$ and $u$ in the underlying function space and the strong convergence is denoted by \(u_n\to u\). 
\(L^q(\Omega;\mathbb R^p)\) denotes the Lebesgue space of \(\mathbb R^p\)-valued functions with respect to the Lebesgue measure on \(\Omega\), while \(L^q(\mu;\mathbb R^p)\) denotes the Lebesgue space with respect to the measure \(\mu\). \(H^1(\Omega;\mathbb R^p)\) denotes the first-order Sobolev space of \(\mathbb R^p\)-valued functions with square-integrable weak derivatives. The codomain \(\mathbb R\) is omitted for scalar-valued function spaces. For an element \(K\) of a Cartesian grid, \(\mathbb Q_1(K)\) denotes the space of affine functions on \(K\).
For a nonempty closed set \(\mathcal C\), its extended-valued indicator function \(\iota_{\mathcal C}\) is defined to be 0 on \(\mathcal C\) and \(+\infty\) outside \(\mathcal C\).
We denote by \(N_{\mathcal C}(x)\) the normal cone to \(\mathcal C\) at \(x\), by \(\partial\phi(x)\) the limiting subdifferential of an extended-real-valued function \(\phi\), and by
\(
\operatorname{dist}(x,\mathcal C) := \inf_{y\in\mathcal C}\|x-y\|
\)
the distance from \(x\) to \(\mathcal C\). For matrices, \(\langle\cdot,\cdot\rangle_F\) and \(\|\cdot\|_F\) denote the Frobenius inner product and norm, respectively. The vector of all ones in \(\mathbb R^n\) is denoted by \(\mathbf 1_n\).


\section{Coupled OT model with landmark constraints}
\label{sec:variational_ot_model} 
In this section, we first review the background of the classical OT model, then introduce our proposed coupled OT model with landmark constraints, followed by its mathematical properties. 

\subsection{Background}
Let \((\mathcal{X},d)\) be a Polish space, and let \(\mathcal P(\Omega)\) denote the space of Borel probability measures on \(\Omega\subset \mathcal{X}\). Given two probability measures \(\mu,\nu\in\mathcal P(\Omega)\), OT aims to find the most cost-efficient way of moving mass of \(\mu\) onto mass distributed according to \(\nu\). Let \(c:\Omega\times\Omega\to [0,+\infty]\) be a lower semicontinuous cost function, the Kantorovich formulation of the OT problem is defined as
\begin{equation}\label{eq:kp_ot}
    \inf_{\pi\in\Pi(\mu,\nu)}
\int c(x,y)\,{\rm d}\pi(x,y),
\end{equation}
where \(\Pi(\mu,\nu)\) denotes the set of all couplings between \(\mu\) and \(\nu\), namely
\[
\Pi(\mu,\nu)
:=
\left\{
\pi\in\mathcal P(\Omega\times\Omega):
\int\pi(x,y){\rm d}y = \mu(x), \int \pi(x,y){\rm d}x = \nu(y)
\right\}.
\]
Any minimizer \(\pi_\star\in\Pi(\mu,\nu)\), when it exists, is called an optimal transport plan. In the special case where \(c(x,y)=d^p(x,y)\), the above problem induces a proper metric in $\mathcal{P}(\Omega)$, namely the \(p\)-Wasserstein distance, given by
\[
W_p(\mu,\nu)
=
\left(
\inf_{\pi\in\Pi(\mu,\nu)}
\int d^p(x,y)\,{\rm d}\pi(x,y)
\right)^{1/p}.
\]
Thus, optimal transport provides both a variational principle for comparing probability measures and a geometric mechanism for identifying how mass should be matched or displaced between them.

\subsection{Coupled OT model}
In the classic OT model \eqref{eq:kp_ot}, the optimal transport plan is determined by the cost function. However, this may cause incorrect transport plans, as the cost function may not reflect the underlying geometrically meaningful deformation field governing the distributional variation. To address this issue, inspired by the idea of sparse-landmark supervision \cite{gu2022keypoint}, we leverage a small number of annotated landmarks to assist the cost minimization criterion for identifying the deformation field. Let $\mathcal S=\{(x_i,y_i)\}_{i=1}^m$ be the set of annotated landmark pairs, where $x_i$ and $y_i$ are observed sample points from $\mu$ and $\nu$, respectively, encoding sparse geometric information about the deformation field. The key challenge is how to integrate this sparse geometric information and the transport cost minimization prior into a principled model to recover the deformation field. We propose the coupled OT framework to tackle this challenge. Our basic idea is to use the landmarks to supervise a deformation field, together with a regularization energy that promotes spatial smoothness and propagates the sparse geometric information across the domain. Finally, the deformation field and the transport plan are coupled within a unified transport model, so that the cost-minimization prior and the sparse landmark supervision can jointly shape the deformation and the transport plan.

For simplicity and ease of exposition, we assume throughout this paper that $\Omega$ is a bounded Lipschitz domain in $\mathbb{R}^p$, while noting that the proposed framework can be naturally extended to more general Banach spaces. For the convenience of description, we denote the landmark displacement by
\[
\Delta_i^{\rm lm}=y_i-x_i,
\]
which provides an observation of the underlying deformation around a prescribed landmark. Let $u:\Omega\to\mathbb{R}^p$ denote the deformation field. The associated transport map is given by
\[
T_u={\rm Id}+u.
\]
In general, we do not require $u$ to be pointwise well-defined; for instance, $u$ may belong only to a Sobolev space. Therefore, landmark information should be interpreted through a bounded observation operator $A_r^i$. A typical choice of $A_r^i$ is the local averaging operator:
$$A_r^i(u)=
\frac{1}{|B_r(x_i)\cap\Omega|}
\int_{B_r(x_i)\cap\Omega}u(z)\dif z,
\text{ for } i=1,\ldots,m,
$$
where $|\cdot|$ is the Lebesgue measure, and $r>0$ is a small radius. 
When the deformation is sufficiently regular and $r$ is small, $A_r^i(u)$ approximates the local displacement around the landmark $x_i$.
We take the admissible deformation space as
$$
\mathcal{U}=\{u\in H^1(\Omega;\mathbb R^p): u=0 \text{ on } \Gamma_D\},
$$
where $\Gamma_D\subset\partial\Omega$ has positive boundary measure. This boundary condition is used to remove rigid-motion ambiguity. We note that other boundary conditions can also be adopted, provided that the corresponding null modes are properly controlled.

We now introduce the coupled OT model with landmark constraints:
\begin{equation}\label{eq:coupled_ot_model}
  \inf_{\pi\in\Pi(\mu,\nu),u\in \mathcal{U}}
\mathcal F(\pi,u),  
\end{equation}
where 
\begin{equation}
\begin{aligned}
\mathcal F(\pi,u) = &\int \left[ (1-\alpha) c(x,y) + \alpha h\left(T_u(x)-y\right)
\right] \dif\pi(x,y)\\
&+ \beta E_{\rm el}(u) + \sum_{i=1}^m w_i\rho\left(A_r^i(u)-\Delta_i^{\rm lm}\right).
\end{aligned}
\end{equation}
Here $\alpha\in(0,1)$, $\beta\geqslant0$ are tuning parameters, and $w_i>0$ are landmark weights. $c(x,y)$ is a lower semicontinuous cost function. The third term incorporates the supervision of sparse landmarks to the deformation field. The
function $\rho$ measures the discrepancy between the observation of the deformation field $A_r^i(u)$ and the true displacement based on landmarks $\Delta_i^{\rm lm}$. 
$\rho:\mathbb R^p\to[0,\infty]$ is a lower semicontinuous function satisfying
\begin{equation}\label{eq:rho_general}
    \rho(0)=0 \text{ and } \inf_{\|z\|_2>\epsilon}\rho(z)>0, \forall \epsilon > 0.
\end{equation} 
Typical choices of $\rho$ include the quadratic loss
$\rho(z)=\|z\|_2^2,$
or some robust loss functions, such as the Huber loss, for cases with noisy landmarks.

The second term is a regularization energy to promote the spatial coherence of the deformation field $u$. Here we use the linear elastic energy, but other regularizers, such as total variation, can also be applied. The linear elastic energy is defined by 
\begin{equation}
\label{eq:elastic}
  E_{\mathrm{el}}(u) = \frac{1}{2} \int \left( \|\varepsilon(u)\|_{F}^{2} +
      (\nabla\cdot u)^{2} \right) \dif x,  
\end{equation}
where
$ \varepsilon(u) = \frac{1}{2} \left( \nabla u + \nabla u^{\top} \right) $
is the linearized strain tensor, and $\nabla\cdot u$ is the divergence of the deformation field. The elastic energy $E_{\mathrm{el}}(u)$ describes the overall local deformation of $u$, including infinitesimal stretching, shearing, and volume change.
By controlling the linearized strain and the divergence of \(u\), this term penalizes local oscillations and promotes a spatially coherent deformation, facilitating the propagation of sparse landmark supervision across the domain.

The first integral term sits in the core part of the proposed model, where $h$ is a lower semicontinuous function for measuring the discrepancy, satisfying 
\begin{equation}\label{eq:h_general}
   h(0)=0 \mbox{ and } \inf_{\|z\|_2>\epsilon}h(z)>0 \text{ for any } \epsilon > 0. 
\end{equation}
Typically, $h$ could be taken as the $p$-squared norm
$h(z)=\|z\|_p^p.$
In this term, for each possibly matched pair \((x,y)\) indicated by the transport plan \(\pi\), we combine the classical transport cost \(c(x,y)\) with the deformation-consistency penalty \(h(T_u(x)-y)\). 
The former encourages mass to be transported along geometrically short paths, while the latter requires the transported target \(y\) to be consistent with the deformation map \(T_u\) applied to the source point \(x\). Therefore, the transport plan \(\pi\) is not only optimized to minimize transport cost but also guided by the landmarks by the consistency with $u$. Meanwhile, it provides a regularization to the optimization of $u$, which enforces $u$ to be consistent with the barycentric projection of $\pi$. 

The following proposition provides a clearer explanation of the coupled OT model specified to the squared Euclidean distance cost.

\begin{proposition}
    Assume $c(x,y)=\Vert x- y\Vert_2^2$ for any $x,y\in\Omega$, $\rho(z)=h(z)=\Vert z\Vert_2^2$ for any $z\in\mathbb{R}^p$.  Let $T_u^{\alpha}={\rm Id} + \alpha u$, then we have
    \begin{equation}
        \inf_{\pi\in\Pi(\mu,\nu),u\in \mathcal{U}} \mathcal F(\pi,u) = \inf_{u\in \mathcal{U}}\mathcal{J}(u)
    \end{equation}
    where 
    \begin{equation}
        \mathcal{J}(u) = W_2^2([T_{u}^{\alpha}]_{\#}\mu,\nu) + \alpha(1-\alpha)\Vert u \Vert^2_{L(\mu;\mathbb R^p)} + \beta E_{\rm el}(u) + \sum_{i=1}^m w_i\rho\left(A_r^i(u)-\Delta_i^{\rm lm}\right),
    \end{equation}
    and $T_{\#} \mu$ is the pushforward measure defined by $T_\# \mu(A)=\mu(T^{-1}(A))$ for any measurable set $A\subset \Omega$.
\end{proposition}
\begin{proof}
First, 
\begin{equation*}
    \begin{split}
       (1-\alpha)c(x,y)+\alpha h(T_u(x)-y) = &(1-\alpha)\Vert x-y\Vert_2^2 + \alpha\Vert x+u(x)-y\Vert_2^2\\
       =&\Vert y-x -\alpha u(x)\Vert_2^2 +  \alpha(1-\alpha)\Vert u(x)\Vert_2^2\\
       =&\Vert y- T_u^{\alpha}\Vert_2^2 + \alpha(1-\alpha)\Vert u(x)\Vert_2^2.
    \end{split}
\end{equation*}
    Thus, \begin{equation*}
    \begin{split}
        & \inf_{\pi\in\Pi(\mu,\nu)} \int  \big[(1-\alpha)c(x,y)+\alpha h(T_u(x)-y)\big]\,\dif\pi \\
        & =  \inf_{\pi\in\Pi(\mu,\nu)} \int  \Vert y-T_u^\alpha(x)\Vert_2^2\,\dif\pi 
        + \alpha(1-\alpha)\Vert u\Vert_{L^2(\mu;{\mathbb {R}}^p)}^2. 
    \end{split}
    \end{equation*}
    For any $\pi\in\Pi(\mu,\nu)$, we define $\gamma_\pi=(T_u^\alpha,{\rm Id})_\#\pi$. 
    Then, we have $\gamma\in\Pi((T_u^\alpha)_\#\mu,\nu)$, and
    $$
    \int 
    \Vert y-T_u^\alpha(x)\Vert_2^2\,\dif\pi(x,y) = \int  \Vert y-z\Vert_2^2\,\dif\gamma_{\pi}(z,y).$$
    Therefore,
    $$
    \inf_{\pi\in\Pi(\mu,\nu)}
    \int 
    \Vert y-T_u^\alpha(x)\Vert_2^2\,\dif\pi(x,y)
    \geqslant
     W_2^2\big((T_u^\alpha)_\#\mu,\nu\big).
    $$
Conversely, for any $\gamma\in\Pi((T_u^\alpha)_\#\mu,\nu)$, set
\(
\eta=({\rm Id},T_u^\alpha)_\#\mu\in\Pi(\mu,(T_u^\alpha)_\#\mu).
\)
By the gluing lemma, there exists $\lambda_{\gamma}\in\mathcal P(\Omega^3)$, with coordinates $(x,z,y)$, such that
\(
(P_{1,2})_\#\lambda_{\gamma}=\eta, (P_{2,3})_\#\lambda_{\gamma}=\gamma.
\)
Here, $P_{i,j}$ denotes the projection from $\Omega^3$ onto its $i$-th and $j$-th coordinates for $i,j=1,2,3$.
Denote $\pi_{\gamma}=(P_{1,3})_\#\lambda_{\gamma}$. Then we have $\pi_{\gamma}\in\Pi(\mu,\nu)$. Since $\eta$ is supported on the graph of $T_u^\alpha$, we have $z=T_u^\alpha(x)$ for $\lambda_\gamma$-a.e. $(x,z,y)$. Hence
\[
\int \|z-y\|_2^2\,\dif\gamma(z,y)
=
\int \|T_u^\alpha(x)-y\|_2^2\,\dif\lambda_{\gamma}(x,z,y)
=
\int \|T_u^\alpha(x)-y\|_2^2\,\dif\pi_{\gamma}(x,y).
\]
Therefore,
\[
\inf_{\pi\in\Pi(\mu,\nu)}
\int \|T_u^\alpha(x)-y\|_2^2\,\dif\pi(x,y)
\leqslant
\int \|z-y\|_2^2\,\dif\gamma(z,y).
\]
Taking the infimum over $\gamma\in\Pi((T_u^\alpha)_\#\mu,\nu)$ gives
\[
\inf_{\pi\in\Pi(\mu,\nu)}
\int \|T_u^\alpha(x)-y\|_2^2\,\dif\pi(x,y)
\leqslant
W_2^2((T_u^\alpha)_\#\mu,\nu).
\] 
Therefore, the equality holds. Since the remaining terms in \(\mathcal F(\pi,u)\) are independent of the transport plan \(\pi\), adding them to both sides gives the asserted result. 
\end{proof}

Proposition 2.1 provides an interpretation of coupled OT. Under the squared $L_2$-cost, the resulting Wasserstein term measures the discrepancy between the target measure $\nu$ and the relaxed pushforward measure \((T_u^\alpha)_\#\mu\), while the additional term \(\alpha(1-\alpha)\|u\|_{L^2(\mu)}^2\) penalizes large displacements. Together with the other terms, the deformation field is optimized through a balanced variational principle: the Wasserstein term provides global distributional guidance, the displacement penalty controls the magnitude of the deformation, the elastic energy promotes spatial coherence, and the landmark loss anchors the deformation at reliable locations. This shows that the coupled OT model turns sparse landmark observations into a dense deformation field by combining distributional alignment with geometric regularization.

\subsection{Theoretical properties}
We first discuss the existence of the solution to the coupled OT model, and then discuss the model behavior when the balancing parameter $\alpha\to 0$ and $\alpha\to 1$. 
\begin{theorem}\label{thm:exsitence}
    Assume that \(\mu\) is absolutely continuous and satisfies $\frac{\dif\mu}{\dif x}\in L^\infty(\Omega)$, and $\beta>0$. Let $\rho$ and $h$ are defined 
    in \eqref{eq:rho_general} and \eqref{eq:h_general}. 
    Suppose there exists \(c_{\rm el}>0\) such that $ E_{\rm el}(u) \geqslant c_{\rm el}\|u\|_{H^1(\Omega;\mathbb R^{p})}^2 \mbox{ for any } u\in \mathcal{U}$, $A_r^i$ is a linear bounded operator, and \(\mathcal F\) is not identically \(+\infty\) on \(\Pi(\mu,\nu)\times\mathcal U\). Then, there exists $(\pi_\star,u_\star)\in\Pi(\mu,\nu)\times\mathcal U$, such that
    $$
    \mathcal F(\pi_\star,u_\star) = \inf_{\pi\in\Pi(\mu,\nu),\,u\in\mathcal U} \mathcal F(\pi,u).
    $$
\end{theorem}
\begin{proof}
Let \(\{(\pi_n,u_n)\}_{n\geqslant 1}\subset\Pi(\mu,\nu)\times\mathcal U\) be a minimizing
sequence to $\mathcal F$.
Then, up to discarding finitely many terms, there exists \(C>0\) such
that $\mathcal F(\pi_n,u_n)\leqslant C $. Since all terms in \(\mathcal F\) are nonnegative, the coercivity of \(E_{\rm el}\) gives
\[
\beta E_{\rm el}(u_n)\leqslant \mathcal F(\pi_n,u_n)\leqslant C,
\text{ and then }
\|u_n\|_{H^1(\Omega;\mathbb R^{p})}^2
\leqslant \frac{C}{\beta c_{\rm el}}.
\]
Thus \(\{u_n\}\) is bounded in \(\mathcal{U}\). Since $\mathcal{U}$ is bounded in $H^1(\Omega;\mathbb R^p)$, there exist a subsequence, still denoted by \(\{u_n\}\), and \(u_\star\in \mathcal{U}\) such that
$ u_n\rightharpoonup u_\star \text{ weakly in }H^1(\Omega;\mathbb R^p). $ Moreover, by the Rellich--Kondrachov compactness theorem,
$ u_n\to u_\star$ strongly in $L^2(\Omega;\mathbb R^p)$.
Using \(\mu\ll dx\) and \(\rho_\mu=:\dif\mu/\dif x\in L^\infty(\Omega)\), we further have
\[
\|u_n-u_\star\|_{L^2(\mu;\mathbb R^p)}^2
=
\int  |u_n-u_\star|^2\rho_\mu\,\dif x
\leqslant
\|\rho_\mu\|_{L^\infty(\Omega)}
\|u_n-u_\star\|_{L^2(\Omega;\mathbb R^p)}^2
\to 0.
\]

Next, we extract a convergent subsequence of transport plans. Since \(\Omega\) is
Polish, \(\mu\) and \(\nu\) are tight.  Hence \(\{\pi_n\}\subset\Pi(\mu,\nu)\) is
tight in \(\mathcal P(\Omega\times\Omega)\). Indeed, for every \(\varepsilon>0\),
there exist compact sets \(K_1,K_2\subset\Omega\) such that
$
\mu(K_1)>1-{\varepsilon}/{2},
\nu(K_2)>1-{\varepsilon}/{2},
$
and therefore
\[
\pi_n((\Omega\times\Omega)\setminus(K_1\times K_2)) \leqslant \mu(\Omega \setminus K_1)+\nu(\Omega \setminus K_2)<\varepsilon.
\]
By Prokhorov's theorem, we have up to a subsequence,
$ \pi_n\rightharpoonup \pi_\star $ narrowly in $\mathcal P(\Omega\times\Omega)$.
The continuity of the marginal projections and  the continuity of pushforwards under narrow convergence indicates $\pi_\star$ has marginals $\mu,\nu$, and hence \(\pi_\star\in\Pi(\mu,\nu)\).

We now prove the lower semicontinuity of \(\mathcal F\). Since \(c\) is nonnegative and
lower semicontinuous, the Portmanteau theorem gives
\[
\int c(x,y)\,\dif\pi_\star
\leqslant
\liminf_{n\to\infty}
\int c(x,y)\,\dif\pi_n .
\]
It remains to handle the deformation-consistency term. Define $G_n(x,y):=x+u_n(x)-y$,
and $G_\star(x,y):=x+u_\star(x)-y$.
Since the first marginal of \(\pi_n\) is \(\mu\),
\[
\int \|G_n-G_\star\|_2^2\,\dif\pi_n
=
\int  \|u_n-u_\star\|_2^2\,\dif\mu
=
\|u_n-u_\star\|_{L^2(\mu;\mathbb R^p)}^2
\to0.
\]
Define mappings $\Phi_n(x,y)=(x,y,G_n(x,y))$ and $\Phi_\star(x,y)=(x,y,G_\star(x,y))$, then $[\Phi_n]_\#\pi_n \rightharpoonup [\Phi_\star]_\#\pi_\star $ narrowly in $\mathcal P(\Omega\times\Omega\times\mathbb R^p)$.
Specifically, since bounded Lipschitz functions determine narrow convergence on Polish spaces, it is enough to test against $\varphi\in BL_b(\Omega\times\Omega\times\mathbb R^p)$.
It can be written that
\begin{equation}\label{eq:narrow_converge}
\begin{split}
    &\int\varphi(x,y,G_n(x,y))\dif \pi_n - \int\varphi(x,y,G_\star(x,y))\dif\pi_\star =  \int \varphi(x,y,G_n(x,y))\dif\pi_n \\
    &-\int \varphi(x,y,G_\star(x,y))\dif \pi_n  + \int \varphi(x,y,G_\star(x,y))\dif \pi_n - \int\varphi(x,y,G_\star(x,y))\dif\pi_\star.
\end{split}
\end{equation}
By the Lipschitz continuity, 
$$ \left|\int \left[\varphi(x,y,G_n(x,y))-\varphi(x,y,G_\star(x,y))\right]\dif \pi_n\right|\leqslant \text{Lip}(\varphi)\int \|G_n-G\|_2\dif\pi_n\rightarrow 0.$$
We tackle the remaining terms in \eqref{eq:narrow_converge} using a Lusin approximation argument.
By Lusin's theorem, for every \(\delta>0\), there exists a compact set
\(K_\delta\subset\Omega\) such that \(\mu(\Omega\setminus K_\delta)<\delta\) and
\(u_\star\) is continuous on \(K_\delta\). By the Tietze extension theorem, there
exists \(v_\delta\in C_b(\Omega;\mathbb R^p)\) such that $v_\delta(x)=u_\star(x), x\in K_\delta$.
We define $\psi(x,y):=\varphi(x,y,G_\star(x,y))$ and $\psi_\delta(x,y):=\varphi(x,y,x+v_\delta(x)-y).$
Then \(\psi_\delta\in C_b(\Omega\times\Omega)\). By the narrow convergence \(\pi_n\rightharpoonup\pi_\star\), we have
\[
\int \psi_\delta\,\dif\pi_n \to \int \psi_\delta\,\dif\pi_\star .
\]
Moreover, since \(\psi=\psi_\delta\) on \(K_\delta\times\Omega\), and
\(\pi_n\) and \(\pi_\star\) have first marginal \(\mu\), we have
\[
\begin{aligned}
\left|
\int (\psi-\psi_\delta)\,\dif\pi_n
\right|
+
\left|
\int (\psi-\psi_\delta)\,\dif\pi_\star
\right|
\leqslant
4\|\varphi\|_\infty\mu(\Omega\setminus K_\delta)
\leqslant
4\|\varphi\|_\infty\delta .
\end{aligned}
\]
Therefore,
\[
\limsup_{n\to\infty}
\left|
\int \psi\,\dif\pi_n
-
\int \psi\,\dif\pi_\star
\right|
\leqslant
4\|\varphi\|_\infty\delta .
\]
Since \(\delta>0\) is arbitrary, we obtain
\[
\int 
\varphi(x,y,G_\star(x,y))\,\dif\pi_n
\to
\int 
\varphi(x,y,G_\star(x,y))\,\dif\pi_\star .
\]
This indicates that $[\Phi_n]_\#\pi_n \rightharpoonup [\Phi_\star]_\#\pi_\star $ narrowly. 
Since \(h\) is nonnegative and lower semicontinuous, the Portmanteau theorem gives
\[
\int 
h(x+u_\star(x)-y)\,\dif\pi_\star
\leqslant
\liminf_{n\to\infty}
\int 
h(x+u_n(x)-y)\,\dif\pi_n .
\]
By the weak lower semicontinuity of \(E_{\rm el}\), we have 
$E_{\rm el}(u_\star) \leqslant \liminf_{n\to\infty}E_{\rm el}(u_n).$
Finally, since each is bounded and linear, weak convergence in \(\mathcal{U}\) implies $A_r^i(u_n)\to A_r^i(u_\star)$ in $\mathbb R^p$. Using the lower semicontinuity of \(\rho\), we obtain
\[
\sum_{i=1}^m w_i\rho\big(A_r^i(u_\star)-\Delta_i^{\rm lm}\big)
\leqslant
\liminf_{n\to\infty}
\sum_{i=1}^m w_i\rho\big(A_r^i(u_n)-\Delta_i^{\rm lm}\big).
\]
Combining the above estimates yields
\begin{equation}\label{eq:minimizer}
   \mathcal F(\pi_\star,u_\star) \leqslant \liminf_{n\to\infty}\mathcal F(\pi_n,u_n). 
\end{equation}
Since $(\pi_n,u_n)$ is a minimizing sequence, the equality holds in \eqref{eq:minimizer}.
\end{proof}

In Theorem~\ref{thm:exsitence}, according to the definition of $\mathcal U$ and by setting $E_{\rm el}$ as the linear elastic energy, the condition that $ E_{\rm el}(u) \geqslant c_{\rm el}\|u\|_{H^1(\Omega;\mathbb R^p)}^2$ for all $u\in \mathcal{U}$ holds for some constant $c_{\rm el}>0$. 
Theorem~\ref{thm:exsitence} guarantees that under mild conditions, the solution of the coupled OT model exists, indicating the problem is well-defined. However, the minimizer may not be unique. Indeed, the transport plan may not be unique in the Kantorovich formulation, and the sparse landmark constraints may not determine a unique deformation field over the whole domain, under the considered conditions.  

In the coupled OT model, the parameter \(\alpha\in(0,1)\) plays a crucial role in balancing the transport cost and the landmark-guided deformation consistency. We next study the model behavior in the two limiting regimes with \(\alpha\to 1\) and \(\alpha\to 0\), respectively.
\begin{remark}\label{thm:prop_alpha_1}
    For any \(\alpha\in(0,1)\), let \(\mathcal F_\alpha\) denote the functional \(\mathcal F\) with the parameter \(\alpha\), and \((\pi_\alpha,u_\alpha)\) be a minimizer of \(\mathcal F_\alpha\). 
    Let $\{\alpha_n\}$ be a sequence such that $\alpha_n\to 1$ as $n\to \infty$. Then, under the conditions of Theorem~\ref{thm:exsitence}, up to a subsequence, $\pi_{\alpha_n}\rightharpoonup \pi^1$ narrowly $\Pi(\mu,\nu)$, and  $u_{\alpha_n}\to u^1$ strongly in $L^2(\mu;\mathbb{R}^p)$. Moreover, 
    \begin{equation}\label{eq:alpha_1}
        (\pi^1,u^1)\in \operatorname*{argmin}_{(\pi,u)\in\Pi(\mu,\nu)\times \mathcal{U}} \left\{R(\pi,u) + J(u)\right\},
    \end{equation}
    where $ J(u) := \beta E_{\rm el}(u) + \sum_{i=1}^m w_i\rho(A_r^i(u)-\Delta_i^{\rm lm}), R(\pi,u) := \int h(T_u(x)-y)\dif\pi(x,y).$ 
\end{remark}

By Remark \ref{thm:prop_alpha_1}, when \(\alpha\to1\), the coupled OT model is stable in the sense that every sequence of minimizers admits a subsequential converge to a solution of the deformation-dominated problem~\eqref{eq:alpha_1}.
This shows that the minimizers remain well-controlled as $\alpha\to 1$.

\begin{remark}\label{thm:prop_alpha_0}
    Let $C(\pi):= \int_{\Omega\times\Omega}c(x,y)\dif\pi(x,y)$, $\mathcal O :=\arg\min_{\pi\in\Pi(\mu,\nu)}C(\pi)$, and $\mathcal M:= \arg\min_{u\in\mathcal U}J(u)$. Then, under the conditions of Theorem~\ref{thm:exsitence}, for every sequence \(\alpha_n\to0\), up to a subsequence, there exists $(\pi^0,u^0)\in \mathcal{O}\times\mathcal{M}$, such that $\pi_{\alpha_n}\rightharpoonup \pi^0$ narrowly $\Pi(\mu,\nu)$, and  $u_{\alpha_n}\to u^0$ strongly in $L^2(\mu;\mathbb{R}^p)$. Moreover, 
    $$(\pi^0,u^0)\in\operatorname*{argmin}_{(\pi,u)\in\mathcal O\times\mathcal M}R(\pi,u).$$
\end{remark}

Remark \ref{thm:prop_alpha_0} shows that, as \(\alpha\to0\), the coupled OT model recovers a classical cost minimization optimal transport plan and an independent elastic-landmark
deformation field. In other words, the  model reduces to the decoupled problem
$$\min_{\pi\in\Pi(\mu,\nu),\,u\in\mathcal U} \left\{C(\pi)+J(u)\right\}.$$
This indicates that the deformation-consistency term \(R(\pi,u)\) does not explicitly affect the limiting energy as \(\alpha\to0\). Rather, it serves as a variational selection criterion among the decoupled minimizers, favoring those that minimize \(R(\pi,u)\).

\section{Computation}
\label{sec:algorithm}
This section discuss the computation of the solution to the coupled OT.
We first devise a variational splitting of the coupled problem at the level of continuous transport plans and Sobolev deformation fields, before introducing a finite-dimensional approximation. 
We then develop a finite-element realization of this principle and discuss the convergence of the resulting scheme.


\subsection{Variational splitting}
Since the coupled OT model~\eqref{eq:coupled_ot_model} involves the optimization of two variables, it is reasonable to perform an alternating update scheme. Specifically, for the $k$-th alternate, $\pi^{(k+1)}$ and $u^{(k+1)}$ is given by 
\begin{equation}
    \begin{split}
        \pi^{(k+1)} &\in \operatorname*{argmin}_{\pi\in\Pi(\mu,\nu)} \mathcal{F}(\pi,u^{(k)}) + \mathcal P_\Pi(\pi,\pi^{(k)}),\\
         u^{(k+1)} &\in \operatorname*{argmin}_{u\in\mathcal U} \mathcal{F}(\pi^{(k+1)},u)+ \mathcal P_{\mathcal U}(u,u^{(k)}),
    \end{split}
\end{equation}
where $\mathcal P_\Pi$ and $\mathcal P_{\mathcal U}$ are stabilization terms for the transport plan and deformation field, respectively. Typically, $\mathcal P_\Pi$ may be chosen as an entropic or Bregman-type regularization for the transport plan, while $\mathcal P_{\mathcal U}$ may be chosen as a quadratic proximal term in the deformation space.

The stabilization terms also provide a basic descent property. Assume that $\mathcal P_\Pi$ and $\mathcal P_{\mathcal U}$ are nonnegative and satisfy $\mathcal P_\Pi(\pi,\pi)=0$ and $\mathcal P_{\mathcal U}(u,u)=0$. Since $\pi^{(k)}$ and $u^{(k)}$ are admissible solutions in the two subproblems, the exact updates imply
\[
\mathcal F(\pi^{(k+1)},u^{(k+1)})+\mathcal P_\Pi(\pi^{(k+1)},\pi^{(k)})+\mathcal P_{\mathcal U}(u^{(k+1)},u^{(k)})\leqslant \mathcal F(\pi^{(k)},u^{(k)}).
\]
Therefore, the objective values $\{\mathcal F(\pi^{(k)},u^{(k)})\}_{k\geqslant 0}$ are nonincreasing. Since $\mathcal F$ is bounded from below under the assumptions of Theorem~\ref{thm:exsitence}, the objective values converge. This descent property, however, does not by itself imply the convergence of the whole sequence $\{(\pi^{(k)},u^{(k)})\}_{k\geqslant 0}$ in the function-space setting. The convergence of the iterates depends on the numerical realization scheme and the stabilization terms.

The above variational splitting is formulated at the function-space level and is not tied to a particular computational representation. In principle, the deformation field and the continuous transport map can be realized by finite elements, splines, radial basis functions, or neural parameterizations, together with a suitable discretization of the input measures. In this work, we adopt a finite-element realization for its compatibility with the Sobolev deformation space and the elastic regularization.

\subsection{Finite-element-based discretization}
\label{sec:discretization}
For the finite-element discretization, we further assume that $\Omega$ is a bounded polyhedral domain. The input measures and the deformation field are discretized separately. This allows the support points of the discrete measures to be independent of the nodes of the finite-element mesh, which is useful to accelerate computation when the measures are given by empirical samples or by quadrature rules.

We first approximate $\mu$ and $\nu$ by the weighted atomic measures
\begin{equation}
\mu_{N_\mu}
=
\sum_{k=1}^{N_\mu} a_k\delta_{\xi_k},
\;
\nu_{N_\nu}
=
\sum_{\ell=1}^{N_\nu} b_\ell\delta_{\eta_\ell},
\label{eq:discrete_measures}
\end{equation}
where $\xi_k,\eta_\ell\in\Omega$, $a_k>0$, $b_\ell>0$, and
\(
\sum_{k=1}^{N_\mu}a_k
=
\sum_{\ell=1}^{N_\nu}b_\ell
=
1.
\)
When \(\mu\) and \(\nu\) are empirical measures, their support points are given by the observed samples, and the corresponding weights are uniform.
For continuously distributed measures, the discrete approximation in \eqref{eq:discrete_measures} can also be constructed from a finite partition of $\Omega$. Specifically, a representative point is chosen from each cell as a support point, and its weight is set to the mass of that cell.
Let
\(
\mathbf a=(a_1,\ldots,a_{N_\mu})^\top,
\mathbf b=(b_1,\ldots,b_{N_\nu})^\top.
\)
A coupling between $\mu_{N_\mu}$ and $\nu_{N_\nu}$ is represented by a nonnegative matrix $\mathbf P\in\mathbb R_+^{N_\mu\times N_\nu}$. The corresponding discrete coupling polytope is
\begin{equation}
\Pi(\mathbf a,\mathbf b)
=
\left\{
\mathbf P\in\mathbb R_+^{N_\mu\times N_\nu}:
\mathbf P\mathbf 1_{N_\nu}=\mathbf a,\;
\mathbf P^\top\mathbf 1_{N_\mu}=\mathbf b
\right\}.
\label{eq:discrete_coupling_set}
\end{equation}
Each $\mathbf P\in\Pi(\mathbf a,\mathbf b)$ induces the discrete transport plan
$\pi_{\mathbf P}
=
\sum_{k=1}^{N_\mu}
\sum_{\ell=1}^{N_\nu}
P_{k\ell}\delta_{(\xi_k,\eta_\ell)}.
\label{eq:discrete_transport_plan}
$

We next approximate the deformation field by finite elements. Let
\(\mathcal T_\delta\) be a conforming finite-element mesh of \(\Omega\),
with mesh size \(\delta\). We consider Cartesian meshes equipped with continuous tensor-product \(\mathbb Q_1\) elements. We note that our approach is applicable to simplicial meshes.
We consider the scalar-valued piecewise affine finite-element space
\[
V_\delta =
\{
v_\delta\in C^0(\overline{\Omega}):
v_\delta|_K\in\mathbb Q_1(K)
\ \text{for every }K\in\mathcal T_\delta,
v_\delta=0 \text{ on }\Gamma_D
\}
\]
and define the vector-valued deformation space by
$\mathcal U_\delta=(V_\delta)^p$. Then we have
$\mathcal U_\delta\subset\mathcal U$.
Let $\{\varphi_q\}_{q=1}^{M_\delta}$ be the nodal basis associated with
the free degrees of freedom of $V_\delta$, where $M_\delta = \dim V_\delta$ denotes the total number of free degrees of freedom. Every
$u_\delta\in\mathcal U_\delta$ can be represented as
\[
u_\delta(x)
=
\sum_{q=1}^{M_\delta}
\varphi_q(x)\mathbf v_q, \mbox{ with }
\mathbf v_q\in\mathbb R^p.
\]
We collect the nodal displacement vectors into
\(
\mathbf v
=
(\mathbf v_1^\top,\ldots,\mathbf v_{M_\delta}^\top)^\top
\in\mathbb R^{pM_\delta},
\)
and define
\(
\mathbf B_\delta(x)
=
[
\varphi_1(x)\mathbf I_p, \) \(\ldots,
\varphi_{M_\delta}(x)\mathbf I_p
]
\in\mathbb R^{p\times pM_\delta}.
\)
It follows that
\[
u_\delta(x)
=
\mathbf B_\delta(x)\mathbf v,
\]
Since the support points $\{\xi_k\}_{k=1}^{N_\mu}$ for measure discretization need not be finite-element nodes, we connect them by evaluating the finite-element deformation field at the support points of source measures:
\begin{equation}
u_\delta(\xi_k)
=
\mathbf B_\delta(\xi_k)\mathbf v,
\mbox{ for }
k=1,\ldots,N_\mu.
\label{eq:deformation_interpolation}
\end{equation}
For piecewise affine finite elements,
\eqref{eq:deformation_interpolation} is the barycentric interpolation of
the nodal displacement vectors on the element containing $\xi_k$.

The elastic energy admits the matrix representation
\begin{equation}
E_{\rm el}(u_\delta)
=
\frac{1}{2}\mathbf v^\top\mathbf K_\delta\mathbf v,
\label{eq:discrete_elastic_energy}
\end{equation}
where $\mathbf K_\delta\in\mathbb R^{pM_\delta\times pM_\delta}$ is the elastic stiffness matrix. 
Let $\{\mathbf e_r\}_{r=1}^p$ be the canonical basis of $\mathbb R^p$ and define the vector-valued finite-element basis $\{\boldsymbol\psi_j\}_{j=1}^{pM_\delta}$ of $\mathcal U_\delta$ by
\(
\boldsymbol\psi_{(q-1)p+r}(x)
=
\varphi_q(x)\mathbf e_r, \text{ for }
q=1,\ldots,M_\delta,
r=1,\ldots,p.
\)
The entries of $\mathbf K_\delta$ are 
\([\mathbf K_\delta]_{ij}
= \mathfrak a_{\rm el}
\bigl(\boldsymbol\psi_i,\boldsymbol\psi_j\bigr),
\text{ for }
i,j=1,\ldots,pM_\delta,
\label{eq:elastic_stiffness_matrix}
\)
where
\begin{equation}
\mathfrak a_{\rm el}(v,w)
=
\int
\left[
\operatorname{tr}\!\left(
\varepsilon(v)^\top\varepsilon(w)
\right)
+
(\nabla\cdot v)(\nabla\cdot w)
\right]\dif x.
\label{eq:elastic_bilinear_form}
\end{equation}
The homogeneous boundary condition on $\Gamma_D$ is incorporated by using only the basis functions associated with the free degrees of freedom.
For each landmark pair $(x_i,y_i)$, let
\( \mathbf A_{i,\delta} \in \mathbb R^{p\times pM_\delta}\) defined by
\begin{equation}
\mathbf A_{i,\delta}
=
\frac{1}{|B_r(x_i)\cap\Omega|}
\int_{B_r(x_i)\cap\Omega}
\mathbf B_\delta(z)\dif z.
\label{eq:local_average_matrix}
\end{equation}
The integral in \eqref{eq:local_average_matrix} can be computed elementwise by standard numerical quadrature.
Then, we have
\begin{equation}
A_r^i(u_\delta)
=
\mathbf A_{i,\delta}\mathbf v.
\label{eq:discrete_landmark_operator}
\end{equation}
We finally discretize the two terms involving the transport plan.
We define the classical transport-cost matrix
$
\mathbf C
=
(C_{k\ell})
\in
\mathbb R^{N_\mu\times N_\nu} \mbox{ with }
C_{k\ell}
=
c(\xi_k,\eta_\ell),
$
and for $\mathbf v\in\mathbb R^{pM_\delta}$, define the
deformation-consistency matrix
$
\mathbf H_\delta(\mathbf v)
=
\bigl(H_{k\ell}(\mathbf v)\bigr)
\in
\mathbb R^{N_\mu\times N_\nu},
$
where
$
H_{k\ell}(\mathbf v)
=
h\left(
\xi_k
+
\mathbf B_\delta(\xi_k)\mathbf v
-
\eta_\ell
\right).
$
The finite-dimensional objective is therefore given by
\begin{equation}
\begin{aligned}
\mathcal F_{\rm dis}(\mathbf P,\mathbf v)
={}&
\left\langle
\mathbf P,
(1-\alpha)\mathbf C
+
\alpha\mathbf H_\delta(\mathbf v)
\right\rangle_F
+
\frac{\beta}{2}
\mathbf v^\top\mathbf K_\delta\mathbf v
\\
&+
\sum_{i=1}^m
w_i
\rho\left(
\mathbf A_{i,\delta}\mathbf v
-
\Delta_i^{\rm lm}
\right).
\end{aligned}
\label{eq:fully_discrete_objective}
\end{equation}
The resulting finite-dimensional
coupled OT problem is
\begin{equation}
\min_{\mathbf P\in\Pi(\mathbf a,\mathbf b),\,
      \mathbf v\in\mathbb R^{pM_\delta}}
\mathcal F_{\rm dis}(\mathbf P,\mathbf v).
\label{eq:fully_discrete_problem}
\end{equation}

\subsection{Alternating iteration}\label{sec:alternating_iteration}
We solve the finite-dimensional problem~\eqref{eq:fully_discrete_problem}
by alternating between the transport plan and the deformation coefficients.
For $\mathbf P,\mathbf Q\in\mathbb R_+^{N_\mu\times N_\nu}$, we define the
generalized Kullback--Leibler divergence by
\begin{equation}
D_\Pi(\mathbf P,\mathbf Q)
=
\sum_{k=1}^{N_\mu}
\sum_{\ell=1}^{N_\nu}
\left[
P_{k\ell}\log\frac{P_{k\ell}}{Q_{k\ell}}
-
P_{k\ell}
+
Q_{k\ell}
\right].
\label{eq:discrete_kl_divergence}
\end{equation}
Given $\lambda_\pi>0$ and $\lambda_v>0$, the discrete alternating
iteration is defined by
\begin{equation}
\left\{
\begin{aligned}
\mathbf P^{(k+1)}
&=
\mathop{\arg\min}_{\mathbf P\in\Pi(\mathbf a,\mathbf b)}
\left\{
\mathcal F_{\rm dis}(\mathbf P,\mathbf v^{(k)})
+
\lambda_\pi
D_\Pi(\mathbf P,\mathbf P^{(k)})
\right\},
\\
\mathbf v^{(k+1)}
&=
\mathop{\arg\min}_{\mathbf v\in\mathbb R^{pM_\delta}}
\left\{
\mathcal F_{\rm dis}(\mathbf P^{(k+1)},\mathbf v)
+
\frac{\lambda_v}{2}
\|\mathbf v-\mathbf v^{(k)}\|_2^2
\right\}.
\end{aligned}
\right.
\label{eq:discrete_alternating_iteration}
\end{equation}
We initialize the transport plan by a strictly positive coupling, namely, $\mathbf P^{(0)}=\mathbf a\mathbf b^\top$.

\textit{Updating the transport plan.}
We define the combined cost matrix
\begin{equation}
\mathbf G(\mathbf v^{(k)})
=
(1-\alpha)\mathbf C
+
\alpha\mathbf H_\delta(\mathbf v^{(k)}).
\label{eq:transport_cost_iteration}
\end{equation}
The first subproblem in~\eqref{eq:discrete_alternating_iteration} can
then be written as
\begin{equation}
\mathbf P^{(k+1)}
=
\mathop{\arg\min}_{\mathbf P\in\Pi(\mathbf a,\mathbf b)}
\left\{\langle
\mathbf P,\mathbf G(\mathbf v^{(k)})\rangle_F
+
\lambda_\pi
D_\Pi(\mathbf P,\mathbf P^{(k)})
\right\}.
\label{eq:transport_plan_subproblem}
\end{equation}
We define the positive kernel
\begin{equation}
\mathbf K^{(k)}
=
\mathbf P^{(k)}
\odot
\exp\left(
-{\mathbf G(\mathbf v^{(k)})}/{\lambda_\pi}
\right),
\label{eq:proximal_sinkhorn_kernel}
\end{equation}
where the exponential is entrywise and $\odot$ is denotes the Hadamard product. By introducing Lagrange multipliers for the marginal
constraints, the first-order optimality conditions of
\eqref{eq:transport_plan_subproblem} imply that
\begin{equation}
\mathbf P^{(k+1)}
=
\operatorname{diag}\bigl(\mathbf s^{(k+1)}\bigr)
\mathbf K^{(k)}
\operatorname{diag}\bigl(\mathbf t^{(k+1)}\bigr),
\label{eq:sinkhorn_scaling_form}
\end{equation}
where
$\mathbf s^{(k+1)}\in\mathbb R_+^{N_\mu}$ and
$\mathbf t^{(k+1)}\in\mathbb R_+^{N_\nu}$ are positive scaling
vectors. Substituting \eqref{eq:sinkhorn_scaling_form} into the
marginal constraints
\(
\mathbf P^{(k+1)}\mathbf 1_{N_\nu}
=
\mathbf a,
(\mathbf P^{(k+1)})^\top\mathbf 1_{N_\mu}
=
\mathbf b,
\)
we obtain a coupled fixed-point system for the scaling vectors. Starting
from a positive vector $\mathbf t^{(k+1,0)}$, it is solved by the
Sinkhorn iteration
\begin{equation}
\left\{
\begin{aligned}
\mathbf s^{(k+1,j+1)}
&=
\mathbf a
\oslash
\left(
\mathbf K^{(k)}
\mathbf t^{(k+1,j)}
\right),
\\
\mathbf t^{(k+1,j+1)}
&=
\mathbf b
\oslash
\left(
(\mathbf K^{(k)})^\top
\mathbf s^{(k+1,j+1)}
\right),
\end{aligned}
\right.
\label{eq:proximal_sinkhorn_iteration}
\end{equation}
where $\oslash$ denotes entrywise division. The limiting scaling
vectors are substituted into \eqref{eq:sinkhorn_scaling_form} to
obtain $\mathbf P^{(k+1)}$. Since $\mathbf P^{(0)}$ is
positive, both $\mathbf P^{(k)}$ and $\mathbf K^{(k)}$ remain
positive at every outer iteration, and the Sinkhorn updates are
well defined.

\paragraph{Updating the deformation} In the numerical realization,
we take \(h(z)=\|z\|_2^2,\) and  \(\rho(z)=\|z\|_2^2\).
For notational simplicity, define
\(
\mathbf B_k=\mathbf B_\delta(\xi_k),
\mathbf A_i=\mathbf A_{i,\delta}, \mbox{ and }
\mathbf d_{k\ell}=\eta_\ell-\xi_k.
\)
The second
subproblem in~\eqref{eq:discrete_alternating_iteration} can be written as
\begin{equation}
\begin{aligned}
\mathbf v^{(k+1)}
=
\mathop{\arg\min}_{\mathbf v\in\mathbb R^{pM_\delta}}
\Bigg\{
&\alpha
\sum_{k=1}^{N_\mu}
\sum_{\ell=1}^{N_\nu}
P_{k\ell}^{(k+1)}
\|\mathbf B_k\mathbf v-\mathbf d_{k\ell}\|_2^2
+
\frac{\beta}{2}
\mathbf v^\top\mathbf K_\delta\mathbf v
\\
&+
\sum_{i=1}^m
w_i
\|\mathbf A_i\mathbf v-\Delta_i^{\rm lm}\|_2^2
+
\frac{\lambda_v}{2}
\|\mathbf v-\mathbf v^{(k)}\|_2^2
\Bigg\}.
\end{aligned}
\label{eq:deformation_coefficient_subproblem}
\end{equation}
We denote 
\begin{equation}
\begin{split}
\mathbf L_\delta
= &
2\alpha
\sum_{k=1}^{N_\mu}
a_k\mathbf B_k^\top\mathbf B_k
+
\beta\mathbf K_\delta
+
2\sum_{i=1}^m
w_i\mathbf A_i^\top\mathbf A_i
+
\lambda_v\mathbf I_{pM_\delta},\\
\mathbf g^{(k+1)}
= &
2\alpha
\sum_{k=1}^{N_\mu}
\sum_{\ell=1}^{N_\nu}
P_{k\ell}^{(k+1)}
\mathbf B_k^\top\mathbf d_{k\ell}
+
2\sum_{i=1}^m
w_i\mathbf A_i^\top\Delta_i^{\rm lm}
+
\lambda_v\mathbf v^{(k)},
\end{split}
\end{equation}
Then, the first-order optimality condition gives the linear system
\begin{equation}
\mathbf L_\delta\mathbf v^{(k+1)}
= 
\mathbf g^{(k+1)}.
\label{eq:deformation_linear_system}
\end{equation}
Here we have used the marginal constraint
$\sum_{\ell=1}^{N_\nu}P_{k\ell}^{(k+1)}=a_k$. 
Since
$\lambda_v>0$, the matrix $\mathbf L_\delta$ is symmetric positive
definite, and~\eqref{eq:deformation_linear_system} admits a unique
solution.

\subsection{Convergence} We now study the convergence of the above iteration algorithm. We first provide the following theorem.
\begin{theorem}[Descent and asymptotic behavior]
\label{thm:descent-asymptotic-behavior}
Assume that \(a_i>0\) for any $i$, \(b_j>0\) for any $j$, \(\beta>0\), and \(\mathbf K_\delta\succ 0\). Suppose that all entries of the discrete cost matrices are finite, the two subproblems are solved exactly, and the transport plan is initialized by \(\mathbf P^{(0)}=\mathbf a\mathbf b^\top\). Then the sequence \(\{(\mathbf P^{(k)},\mathbf v^{(k)})\}_{k\geqslant 0}\) generated by the discrete alternating iteration satisfies the following properties:
\begin{enumerate}
\item The objective sequence
$\left\{\mathcal F_{\rm dis}(\mathbf P^{(k)},\mathbf v^{(k)})\right\}_{k\geq 0}
$
is non-increasing and converges to a finite limit \(\mathcal F_\infty\).
\item The accumulated stabilization terms are finite:
\[
\sum_{k=0}^{\infty}D_\Pi(\mathbf P^{(k+1)},\mathbf P^{(k)})<\infty
\mbox{ and }
\sum_{k=0}^{\infty}\|\mathbf v^{(k+1)}-\mathbf v^{(k)}\|_2^2<\infty.
\]
\item The differences between successive iterates vanish:
\[
\|\mathbf P^{(k+1)}-\mathbf P^{(k)}\|_F\rightarrow 0,
\mbox{ and }
\|\mathbf v^{(k+1)}-\mathbf v^{(k)}\|_2\rightarrow 0.
\]

\item The sequence \(\{(\mathbf P^{(k)},\mathbf v^{(k)})\}_{k\geqslant 0}\) is bounded. Its cluster set
is nonempty, compact, and connected.

\item Every cluster point \((\overline{\mathbf P},\overline{\mathbf v})\) satisfies
\[
\nabla_{\mathbf v}\mathcal F_{\rm dis}(\overline{\mathbf P},\overline{\mathbf v})=\mathbf 0.
\]
\end{enumerate}
\end{theorem}

\begin{proof}
Since \(\mathbf P^{(k)}\in\Pi(\mathbf a,\mathbf b)\) is an admissible competitor for the transport-plan subproblem, the optimality of \(\mathbf P^{(k+1)}\) gives
\[
\mathcal F_{\rm dis}(\mathbf P^{(k+1)},\mathbf v^{(k)})+\lambda_\pi D_\Pi(\mathbf P^{(k+1)},\mathbf P^{(k)})\leqslant\mathcal F_{\rm dis}(\mathbf P^{(k)},\mathbf v^{(k)}).
\]
Similarly, we have 
\[
\mathcal F_{\rm dis}(\mathbf P^{(k+1)},\mathbf v^{(k+1)})+\frac{\lambda_v}{2}\|\mathbf v^{(k+1)}-\mathbf v^{(k)}\|_2^2\leqslant\mathcal F_{\rm dis}(\mathbf P^{(k+1)},\mathbf v^{(k)}).
\]
Combining these two inequalities yields
\begin{equation}
\begin{aligned}
&\mathcal F_{\rm dis}(\mathbf P^{(k)},\mathbf v^{(k)})-\mathcal F_{\rm dis}(\mathbf P^{(k+1)},\mathbf v^{(k+1)}) \\
&\qquad\geqslant\lambda_\pi D_\Pi(\mathbf P^{(k+1)},\mathbf P^{(k)})+\frac{\lambda_v}{2}\|\mathbf v^{(k+1)}-\mathbf v^{(k)}\|_2^2\geqslant 0.
\end{aligned}
\label{eq:discrete-descent}
\end{equation}
Therefore, the objective sequence is nonincreasing. Since \(\mathcal F_{\rm dis}\) is bounded from below, there exists \(\mathcal F_\infty\in\mathbb R\) such that \(\mathcal F_{\rm dis}(\mathbf P^{(k)},\mathbf v^{(k)})\longrightarrow\mathcal F_\infty.\)

Summing \eqref{eq:discrete-descent} from \(k=0\) to \(K\), we obtain
\begin{equation}
\begin{aligned}
&\lambda_\pi\sum_{k=0}^{K}D_\Pi(\mathbf P^{(k+1)},\mathbf P^{(k)})+\frac{\lambda_v}{2}\sum_{k=0}^{K}\|\mathbf v^{(k+1)}-\mathbf v^{(k)}\|_2^2 \\
&\qquad\leqslant\mathcal F_{\rm dis}(\mathbf P^{(0)},\mathbf v^{(0)})-\mathcal F_{\rm dis}(\mathbf P^{(K+1)},\mathbf v^{(K+1)}).
\end{aligned}
\label{eq:discrete-summability-bound}
\end{equation}
Let \(K\to\infty\), we obtain
\[
\sum_{k=0}^{\infty}D_\Pi(\mathbf P^{(k+1)},\mathbf P^{(k)})<\infty
\mbox{ and }
\sum_{k=0}^{\infty}\|\mathbf v^{(k+1)}-\mathbf v^{(k)}\|_2^2<\infty.
\]
Consequently, we have 
\[
D_\Pi(\mathbf P^{(k+1)},\mathbf P^{(k)})\rightarrow 0
\mbox{ and }
\|\mathbf v^{(k+1)}-\mathbf v^{(k)}\|_2\rightarrow 0.
\]

Since \(\mathbf P^{(k+1)}\) and \(\mathbf P^{(k)}\) both have total mass one, Pinsker's inequality gives
\[
D_\Pi(\mathbf P^{(k+1)},\mathbf P^{(k)})\geqslant\frac{1}{2}\|\mathbf P^{(k+1)}-\mathbf P^{(k)}\|_1^2.
\]
Thus,
\[
\|\mathbf P^{(k+1)}-\mathbf P^{(k)}\|_F\leqslant\|\mathbf P^{(k+1)}-\mathbf P^{(k)}\|_1\rightarrow 0.
\]

We next prove the boundedness of the iteration sequence. Let $\mathbf z^{(k)}=(\mathbf P^{(k)},\mathbf v^{(k)})$, then we have 
\begin{equation}
\|\mathbf z^{(k+1)}-\mathbf z^{(k)}\|\rightarrow 0,
\label{eq:asymptotic-regularity}
\end{equation}
where
\[
\|\mathbf z^{(k+1)}-\mathbf z^{(k)}\|^2
:=
\|\mathbf P^{(k+1)}-\mathbf P^{(k)}\|_F^2
+
\|\mathbf v^{(k+1)}-\mathbf v^{(k)}\|_2^2.
\]
The transport polytope \(\Pi(\mathbf a,\mathbf b)\) is compact, so \(\{\mathbf P^{(k)}\}_{k\geqslant 0}\) is bounded. Additionally, 
\[
\mathcal F_{\rm dis}(\mathbf P,\mathbf v)\geqslant\frac{\beta}{2}\mathbf v^\top\mathbf K_\delta\mathbf v.
\]
Since \(\mathbf K_\delta\succ 0\), there exists \(c_\delta>0\) such that
\[
\mathbf v^\top\mathbf K_\delta\mathbf v\geqslant c_\delta\|\mathbf v\|_2^2
\]
for every \(\mathbf v\in\mathbb R^{pM_\delta}\). Therefore,
\[
\frac{\beta c_\delta}{2}\|\mathbf v^{(k)}\|_2^2\leqslant\mathcal F_{\rm dis}(\mathbf P^{(k)},\mathbf v^{(k)})\leqslant\mathcal F_{\rm dis}(\mathbf P^{(0)},\mathbf v^{(0)}).
\]
Thus, \(\{\mathbf v^{(k)}\}_{k\geqslant 0}\) is bounded, and hence \(\{\mathbf z^{(k)}\}_{k\geqslant 0}\) is bounded. If we denote the cluster set of $\{\mathbf z^{(k)}\}_{k\geqslant 0}$ as $$\omega(\mathbf z^{(0)}):=\{\overline{\mathbf z}:\mathbf z^{(k_j)}\to\overline{\mathbf z}\text{ for some }k_j\to\infty\},$$ then $\omega(\mathbf z^{(0)})$ is nonempty and compact. Let \(\mathbf z^{(k_j)}\to\overline{\mathbf z}\) be any convergent subsequence. By the continuity of \(\mathcal F_{\rm dis}\), we have
\[
\mathcal F_{\rm dis}(\overline{\mathbf z})=\lim_{j\to\infty}\mathcal F_{\rm dis}(\mathbf z^{(k_j)})=\mathcal F_\infty.
\]
Thus, all cluster points have the same objective value. 

It remains to prove that the cluster set is connected. Suppose, by contradiction,
\[
\omega(\mathbf z^{(0)})=\mathcal A\cup\mathcal B,
\]
where \(\mathcal A\) and \(\mathcal B\) are nonempty, disjoint compact sets. Then
\(d:=\operatorname{dist}(\mathcal A,\mathcal B)>0.\)
Since \(\mathcal A\) and \(\mathcal B\) are nonempty subsets of the
cluster set, choose \(\mathbf a\in\mathcal A\) and
\(\mathbf b\in\mathcal B\). By the definition of a cluster point, the
sequence has subsequences converging to \(\mathbf a\) and \(\mathbf b\),
respectively. Hence, it visits the \(d/3\)-neighborhoods of both
\(\mathcal A\) and \(\mathcal B\) infinitely many times. By \eqref{eq:asymptotic-regularity}, for all sufficiently large \(k\),
\[
\|\mathbf z^{(k+1)}-\mathbf z^{(k)}\|<{d}/{3}.
\]
Therefore, whenever the sequence moves from the \(d/3\)-neighborhood of \(\mathcal A\) to the \(d/3\)-neighborhood of \(\mathcal B\), it must contain an intermediate iterate outside both neighborhoods. Since the sequence is bounded, these intermediate iterates admit a cluster point outside \(\mathcal A\cup\mathcal B\), which is a contradiction. Hence, the cluster set is connected.

Finally, the first-order optimality condition for the deformation update is
\begin{equation}
\nabla_{\mathbf v}\mathcal F_{\rm dis}(\mathbf P^{(k+1)},\mathbf v^{(k+1)})+\lambda_v(\mathbf v^{(k+1)}-\mathbf v^{(k)})=\mathbf 0.
\label{eq:deformation-optimality}
\end{equation}
Let
\(
(\mathbf P^{(k_j)},\mathbf v^{(k_j)})\rightarrow(\overline{\mathbf P},\overline{\mathbf v})
\)
be a subsequence converging to a cluster point. By \eqref{eq:asymptotic-regularity}, we have
\(
(\mathbf P^{(k_j+1)},\mathbf v^{(k_j+1)})\rightarrow(\overline{\mathbf P},\overline{\mathbf v}).
\)
Passing to the limit in \eqref{eq:deformation-optimality} and using the continuity of \(\nabla_{\mathbf v}\mathcal F_{\rm dis}\), we obtain
\[
\nabla_{\mathbf v}\mathcal F_{\rm dis}(\overline{\mathbf P},\overline{\mathbf v})=\mathbf 0.
\]
\end{proof}

In Theorem~\ref{thm:descent-asymptotic-behavior}, the assumptions are naturally satisfied by the present discretization. In particular, the coercivity of the discrete elastic energy under the imposed boundary conditions ensures that \(\mathbf K_\delta\) is symmetric positive definite.
Theorem~\ref{thm:descent-asymptotic-behavior} establishes the asymptotic properties of the discrete alternating iteration. Specifically, the objective value decreases monotonically and converges to a finite limit, while the accumulated stabilization terms remain finite. Consequently, the differences between successive iterates vanish, ruling out persistent oscillations with nonvanishing step sizes. Moreover, the generated sequence is bounded, and its cluster set is nonempty, compact, and connected. Every cluster point is stationary with respect to the deformation variable. These results provide a stability guarantee for the iteration.

By Theorem~\ref{thm:descent-asymptotic-behavior}, if the cluster set is finite, its connectedness implies that it consists of a single point. Consequently, the entire iteration sequence converges to this point. However, such a condition is hard to verify. Next, we incorporate an entropy regularization term for the transport plan into the objective function, which allows us to establish convergence of the entire iteration sequence.



\paragraph{Entropy-regularized iteration}
We define the entropy
\[
\operatorname{Ent}(\mathbf P)
:=
-\sum_{i=1}^{N_\mu}
\sum_{j=1}^{N_\nu}
P_{ij}(\log P_{ij}-1),
\]
and the objective becomes $\mathcal{F}_{\rm dis}-\lambda_{\rm ent}\operatorname{Ent}(\mathbf P)$ for some \(\lambda_{\rm ent}>0\). The corresponding alternating iteration is
\begin{equation}
\left\{
\begin{aligned}
\mathbf P^{(k+1)}
&=
\mathop{\arg\min}_{\mathbf P\in\Pi(\mathbf a,\mathbf b)}
\left\{
\mathcal F_{\rm dis}(\mathbf P,\mathbf v^{(k)})
-
\lambda_{\rm ent}\operatorname{Ent}(\mathbf P)
\right\},
\\
\mathbf v^{(k+1)}
&=
\mathop{\arg\min}_{\mathbf v\in\mathbb R^{pM_\delta}}
\left\{
\mathcal F_{\rm dis}(\mathbf P^{(k+1)},\mathbf v)
+
\frac{\lambda_v}{2}
\|\mathbf v-\mathbf v^{(k)}\|_2^2
\right\}.
\end{aligned}
\right.
\label{eq:entropy_alternating}
\end{equation}
where we omitted the stabilization $D_{\Pi}$ for updating $\mathbf{P}$.  
Consequently, we only need to change the definition in \eqref{eq:proximal_sinkhorn_kernel} to 
\[
\mathbf K^{(k)}
=
\exp\left(
-{\mathbf G(\mathbf v^{(k)})}/{\lambda_{\rm ent}}
\right),
\]
while all subsequent steps remain unchanged.

\begin{theorem}\label{cor:entropy_convergence}
Assume that \(a_i>0\) for any $i$, \(b_j>0\) for any $j$, \(\beta>0\), and \(\mathbf K_\delta\succ 0\), and that all discrete costs are finite. Suppose that the subproblems in \eqref{eq:entropy_alternating} are solved exactly. Then
\[
(\mathbf P^{(k)},\mathbf v^{(k)})
\rightarrow
(\mathbf P^\star,\mathbf v^\star),
\]
where \((\mathbf P^\star,\mathbf v^\star)\) is a constrained critical
point of
\(
\mathcal F_{\rm dis}
-\lambda_{\rm ent}\operatorname{Ent}(\mathbf{P})
\)
on
\(\Pi(\mathbf a,\mathbf b)\times\mathbb R^{pM_\delta}\).
\end{theorem}

\begin{proof}
We denote the extended objective by
\[
\Phi(\mathbf P,\mathbf v)
:=
\mathcal F_{\rm dis}(\mathbf P,\mathbf v)
-
\lambda_{\rm ent}\operatorname{Ent}(\mathbf P)
+
\iota_{\Pi(\mathbf a,\mathbf b)}(\mathbf P).
\]
The optimality condition of the transport-plan update, together with
the Bregman identity for the entropy, gives
\[
\Phi(\mathbf P^{(k)},\mathbf v^{(k)})
-
\Phi(\mathbf P^{(k+1)},\mathbf v^{(k)})
\geqslant
\lambda_{\rm ent}
D_\Pi(\mathbf P^{(k)},\mathbf P^{(k+1)}).
\]
The optimality of the deformation update gives
\[
\begin{aligned}
\Phi(\mathbf P^{(k+1)},\mathbf v^{(k)})
-
\Phi(\mathbf P^{(k+1)},\mathbf v^{(k+1)})
\geqslant
\frac{\lambda_v}{2}
\|\mathbf v^{(k+1)}-\mathbf v^{(k)}\|_2^2.
\end{aligned}
\]
Combining these inequalities with Pinsker's inequality yields, for some
\(c_1>0\),
\[
\Phi(\mathbf z^{(k)})
-
\Phi(\mathbf z^{(k+1)})
\geqslant
c_1
\|\mathbf z^{(k+1)}-\mathbf z^{(k)}\|^2.
\]
Since \(\Pi(\mathbf a,\mathbf b)\) is compact,
\(-\operatorname{Ent}\) is bounded from below on
\(\Pi(\mathbf a,\mathbf b)\). The decrease of
\(\Phi(\mathbf z^{(k)})\), together with
\(\beta>0\) and \(\mathbf K_\delta\succ0\), implies that
\(\{\mathbf v^{(k)}\}_{k\geqslant0}\) is bounded. By the continuity of
\(\mathbf G\), there exist constants
$
0<\kappa^-
\leqslant
\kappa^+
<\infty
$
such that
\[
\kappa^-
\leqslant
[\mathbf K^{(k)}]_{ij}
\leqslant
\kappa^+
\text{ for all }i,j,k.
\]
Let
\(
S^{(k+1)}
:=
\sum_i s_i^{(k+1)},
T^{(k+1)}
:=
\sum_j t_j^{(k+1)}.
\)
The marginal constraints imply
\[
s_i^{(k+1)}
\geqslant
\frac{a_i}{\kappa^+T^{(k+1)}},
\;
t_j^{(k+1)}
\geqslant
\frac{b_j}{\kappa^+S^{(k+1)}}.
\]
Moreover, since \(\mathbf P^{(k+1)}\) has unit total mass,
\[
1
=
\sum_{i,j}
s_i^{(k+1)}
[\mathbf K^{(k)}]_{ij}
t_j^{(k+1)}
\geqslant
\kappa^-
S^{(k+1)}T^{(k+1)}.
\]
It follows that
\[
P_{ij}^{(k+1)}
\geqslant
a_i b_j
\left(
{\kappa^-}/
{\kappa^+}
\right)^2
\text{ for all }i,j,k.
\]
Hence the transport iterates are uniformly bounded away from zero.
The optimality condition of the transport-plan update gives
\[
\mathbf 0
\in
\mathbf G(\mathbf v^{(k)})
+
\lambda_{\rm ent}\log\mathbf P^{(k+1)}
+
N_{\Pi(\mathbf a,\mathbf b)}
(\mathbf P^{(k+1)}).
\]
Since
\[
\partial_{\mathbf P}
\Phi(\mathbf z^{(k+1)})
=
\mathbf G(\mathbf v^{(k+1)})
+
\lambda_{\rm ent}\log\mathbf P^{(k+1)}
+
N_{\Pi(\mathbf a,\mathbf b)}
(\mathbf P^{(k+1)}),
\]
we obtain
\[
\operatorname{dist}
\left(
\mathbf 0,
\partial_{\mathbf P}
\Phi(\mathbf z^{(k+1)})
\right)
\leqslant
\|\mathbf G(\mathbf v^{(k+1)})
-
\mathbf G(\mathbf v^{(k)})\|_F.
\]
Since \(\{\mathbf v^{(k)}\}_{k\geqslant0}\) is bounded and
\(\mathbf G\) is continuously differentiable, there exists \(L_G>0\)
such that
\[
\operatorname{dist}
\left(
\mathbf 0,
\partial_{\mathbf P}
\Phi(\mathbf z^{(k+1)})
\right)
\leqslant
L_G
\|\mathbf v^{(k+1)}-\mathbf v^{(k)}\|_2.
\]
For the deformation component, the optimality condition gives
\[
\nabla_{\mathbf v}
\Phi(\mathbf z^{(k+1)})
=
-\lambda_v
(\mathbf v^{(k+1)}-\mathbf v^{(k)}).
\]
Consequently, there exists \(c_2>0\) such that
\[
\operatorname{dist}
\left(
\mathbf 0,
\partial\Phi(\mathbf z^{(k+1)})
\right)
\leqslant
c_2
\|\mathbf z^{(k+1)}-\mathbf z^{(k)}\|_2.
\]
The uniform positivity of \(\{\mathbf P^{(k)}\}_{k\geqslant 1}\) implies that there exists \(\tau>0\) such that
\(
P_{ij}^{(k)}\geqslant\tau
\)
for all \(i,j\) and \(k\geqslant 1\). Consequently, every cluster point \(\bar{\mathbf P}\) of the transport iterates satisfies \(\bar P_{ij}\geqslant\tau\). Hence, in a sufficiently small neighborhood of each cluster point, all entries of \(\mathbf P\) remain strictly positive, and the nonnegativity constraints in \(\Pi(\mathbf a,\mathbf b)\) are inactive. Locally, the feasible set therefore coincides with the affine subspace determined by
\[
\mathbf P\mathbf 1_{N_{\nu}}=\mathbf a,\;
\mathbf P^\top\mathbf 1_{N_{\mu}}=\mathbf b.
\]
On this neighborhood, the entropy term $-{\rm Ent}(\mathbf P)$is real analytic. Moreover, under the quadratic choices of \(h\) and \(\rho\), the function \(\mathcal F_{\rm dis}\) is polynomial in \((\mathbf P,\mathbf v)\). Therefore, the restriction of the smooth part of \(\Phi\) to the above affine subspace is real analytic and satisfies the \L{}ojasiewicz gradient inequality. Equivalently, \(\Phi\) satisfies the Kurdyka--\L{}ojasiewicz property in a neighborhood of every cluster point of the sequence.
The preceding
sufficient-descent and subgradient estimates, together with the
boundedness of \(\{\mathbf z^{(k)}\}_{k\geqslant0}\), therefore imply
\[
\sum_{k=0}^{\infty}
\|\mathbf z^{(k+1)}-\mathbf z^{(k)}\|_2
<
\infty
\]
by the Kurdyka--\L{}ojasiewicz convergence theorem. Hence
\(\{\mathbf z^{(k)}\}_{k\geqslant0}\) is a Cauchy sequence and converges
to some \(\mathbf z^\star\). Moreover, the subgradient estimate gives
\[
\operatorname{dist}
\left(
\mathbf 0,
\partial\Phi(\mathbf z^{(k)})
\right)
\rightarrow
0.
\]
By the closedness of the limiting subdifferential,
\[
\mathbf 0
\in
\partial\Phi(\mathbf z^\star).
\]
Thus, \(\mathbf z^\star=(\mathbf P^\star,\mathbf v^\star)\) is a
constrained critical point of $\mathcal F_{\rm dis} -\lambda_{\rm ent}{\rm Ent}(\mathbf{P})$.
\end{proof}


By Theorem~\ref{cor:entropy_convergence}, the iteration sequence for the entropy-regularized discrete objective converges to a constrained critical point. 

\section{Numerical results}
We first conduct controlled experiments on synthetic fish-shaped distributions with explicitly prescribed ground-truth deformation fields. We then present a real shape-matching experiment for which the ground-truth deformation field is unavailable.
\subsection{Synthetic experiments}
\label{sec:fish}


We set the computational domain as \(\Omega=[0,1]^2\). The source distribution is constructed from a filled fish-shaped mask (see Fig.~\ref{fig:density_landmark}). The binary mask is smoothed using a Gaussian filter with standard deviation \(0.8\). The resulting density is normalized to have unit mass. The target distribution is generated by deforming the source distribution through the prescribed map $u^\star$.
To make the ground-truth deformation compatible with the homogeneous Dirichlet boundary condition, we introduce the cutoff function
\[
\chi_\tau(x)=\zeta_\tau(x_1)\zeta_\tau(x_2), \quad \tau=0.05,
\]
where
\[
\zeta_\tau(s)=
\begin{cases}
s/\tau, & 0\leqslant s<\tau,\\
1, & \tau\leqslant s\leqslant 1-\tau,\\
(1-s)/\tau, & 1-\tau<s\leqslant 1.
\end{cases}
\]
The ground-truth deformation is then defined by
\[
u^\star(x)
=
\chi_\tau(x)
\left[
(A-I_2)(x-x_{\mathrm c})+u_{\mathrm{nr}}(x)
\right],
\quad
x_{\mathrm c}=(0.5,0.5)^\top,
\]
where
\[
A
=
R(8)\operatorname{diag}(1.08,0.94),
\quad
R(\theta)
=
\begin{pmatrix}
\cos\theta & -\sin\theta\\
\sin\theta & \cos\theta
\end{pmatrix}.
\]
For \(x=(x_1,x_2)^\top\), the nonrigid component is defined as
\[
u_{\mathrm{nr}}(x)
=
\begin{pmatrix}
0.018\sin(\pi x_1)\sin(2\pi x_2)+0.012q(x)\\
0.020\sin(2\pi x_1)\sin(\pi x_2)-0.010(x_1-0.55)q(x)
\end{pmatrix},
\]
where
\[
q(x)
=
\exp\!\left(
-
\Bigl(\frac{x_1-0.66}{0.16}\Bigr)^2
-
\Bigl(\frac{x_2-0.54}{0.11}\Bigr)^2
\right).
\]
Here, \(\pi\) denotes the circle constant. This deformation combines rotation, anisotropic scaling, smooth spatial oscillation, and a localized nonlinear deformation near the front part of the fish, providing a nontrivial test case involving multiple deformation patterns; it also satisfies \(u^\star=0\) on \(\partial\Omega\). The target density is obtained by transporting the source masses through \(T^\star\) and interpolating them back onto the original Cartesian grid using barycentric coordinates.
A Gaussian filter with standard deviation \(0.6\) is subsequently applied, followed by mass normalization.
The source and target measures are discretized on a uniform \(11\times11\) Cartesian grid, as computing the transport plan on a larger grid is expensive. The deformation field is represented on a uniform \(21\times21\) Cartesian grid.  
Eight candidate landmarks are placed at representative positions of the source fish, including the nose, the upper and lower tail tips, the dorsal and ventral fin tips, the upper and lower body boundaries, and the tail base, as illustrated in Fig.~\ref{fig:density_landmark}. For a source landmark \(x_i\), its target counterpart \(y_i\) is generated using the local observation operator \(A_r^i(u^\star)\) on $x_i$.
\begin{figure}
    \centering
    \includegraphics[width=0.7\linewidth]{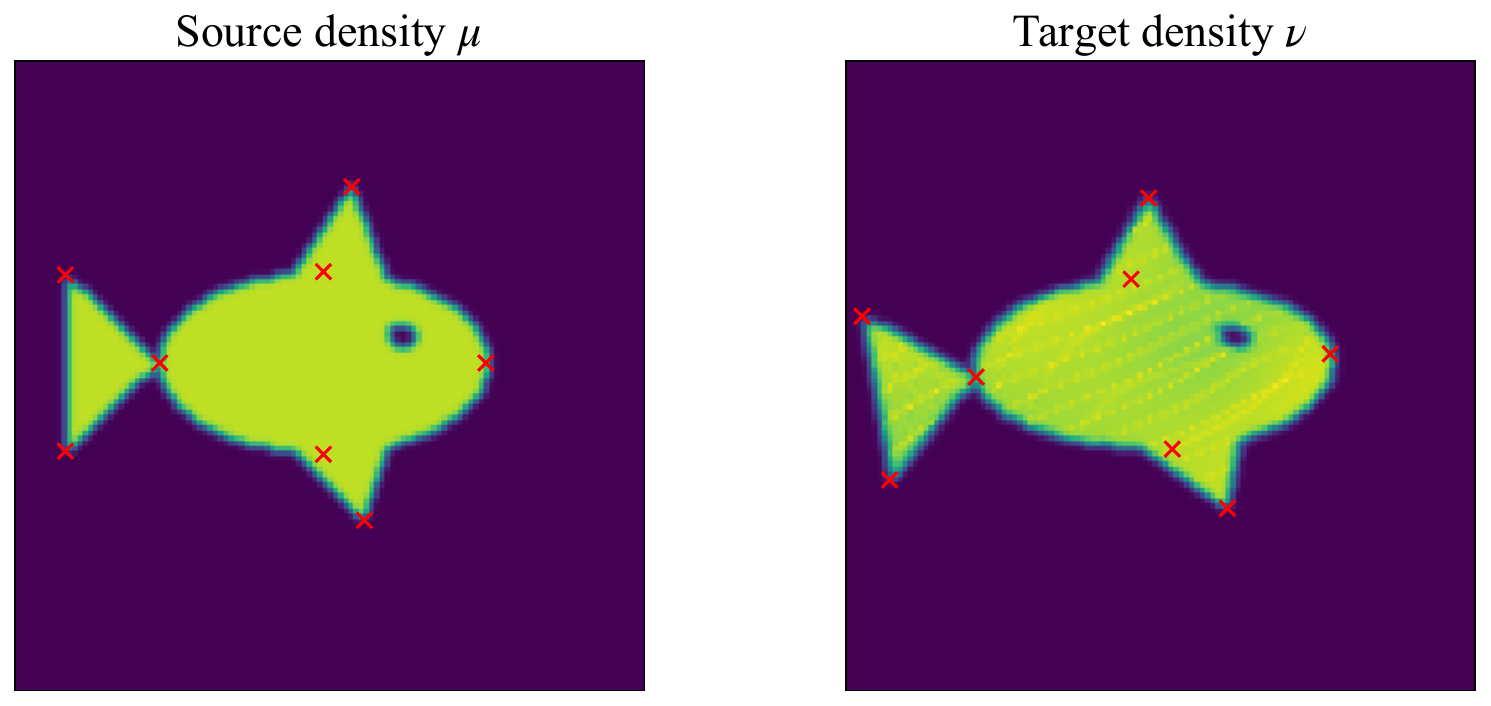}
    \caption{Fish-shaped source and target densities, together with the candidate landmarks marked by red crosses.}
    \label{fig:density_landmark}
\end{figure}

To validate the effectiveness of our coupled OT (denoted by COT), we compare it with two baselines: OT (barycentric) and Landmark-only. The OT (barycentric) method estimates the deformation via barycentric projection of an OT plan. The Landmark-only method estimates the deformation from landmark fitting and elastic regularization alone.
Our proposed COT is implemented using both the KL-based proximal and entropy-regularized algorithms; we denote them by COT (proximal KL) and COT (entropy), respectively.

All numerical experiments are conducted in a single-threaded setting using Python 3.9 on a CentOS 7 Linux server equipped with two Intel Xeon Gold 6246R processors running at \(3.40\) GHz and \(376\) GB of RAM. We use the quadratic choices
\(
c(x,y)=\|x-y\|_2^2,
h(z)=\|z\|_2^2,
\rho(z)=\|z\|_2^2,
\)
and assign the same weight \(w_i=1\) to all observed landmarks. The model parameters are set as
\(\alpha=0.5,\beta=0.05.\)
For the KL-proximal algorithm, we take
\(\lambda_\pi=10^{-3},\lambda_v=10^{-3},\)
whereas the entropy parameter in the entropy-regularized algorithm is set as
\(
\lambda_{\mathrm{ent}}=5\times10^{-3}.
\)
The transport plan and deformation coefficients are initialized by
\(
\mathbf P^{(0)}=\mathbf a\mathbf b^\top,\mathbf v^{(0)}=\mathbf 0.
\)

Since the ground-truth deformation is available, we first assess the recovered field using the \textit{Deformation error} on the source distribution
\[
e_{L^2(\mu)}
=
\big(
\sum_j
a_j
\left\|
u_\delta(\xi_j)-u^\star(\xi_j)
\right\|_2^2
\big)^{1/2}.
\]
In the experiments, we randomly hold a subset of the candidate landmarks out for computation and evaluate the method on the held-out landmarks. The \textit{Held-out landmark error} is defined by 
\[
e_{\mathrm{lm}}^{\mathrm{held}}
=
\big(
\frac{1}{|\mathcal I_{\mathrm{held}}|}
\sum_{i\in\mathcal I_{\mathrm{held}}}
\left\|
\mathbf{A}_{i,\delta}\mathbf v-A_r^i(u^\star)
\right\|_2^2
\big)^{1/2},
\]
where $\mathcal I_{\mathrm{held}}$ is set of the held-out landmarks.
To measure distributional alignment, we warp the source density using the estimated deformation and compute its discrete \(L^1\) discrepancy from the target density (\textit{Density error}). All the metrics are computed on the $21\times21$ fine grid.

\begin{figure}
    \centering
    \includegraphics[width=1\linewidth]{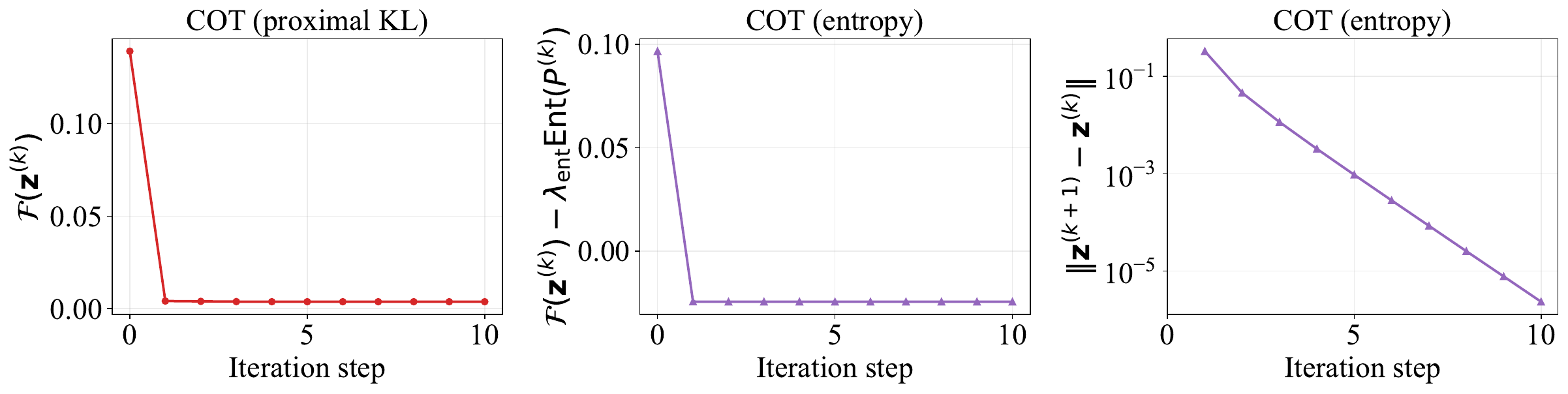}
    \caption{Left and middle: objective values $\mathcal{F}(\mathbf{z}^{(k)})$ along the iterations of the KL-proximal (COT (proximal KL)) and entropy-regularized (COT (entropy)) algorithms for coupled OT. Right: relative error $\|\mathbf{z}^{(k+1)}-\mathbf{z}^{(k)}\|$ between two successive iterates of COT (entropy).}
    \label{fig:convergence}
\end{figure}

\paragraph{Convergence} We first validate the convergence properties of the computational algorithms. For the KL-based proximal algorithm, we plot the objective function values along the iteration sequence in the left panel of Fig.~\ref{fig:convergence}. We can observe that the objective function value stably decreases and converges, which echoes Theorem~\ref{thm:descent-asymptotic-behavior}. For the entropy-regularized algorithm, we theoretically show the convergence of the iteration sequence in Theorem~\ref{cor:entropy_convergence}. In Fig.~\ref{fig:convergence}, we plot both the objective function values along the iteration sequence in the middle panel and the relative error between two successive iterates in the right panel. 
It can be seen that the objective function values decrease and the iteration sequence converges, as the relative error stably decreases; this echoes Theorem~\ref{cor:entropy_convergence}.

\begin{figure}
    \centering
    \includegraphics[width=1\linewidth]{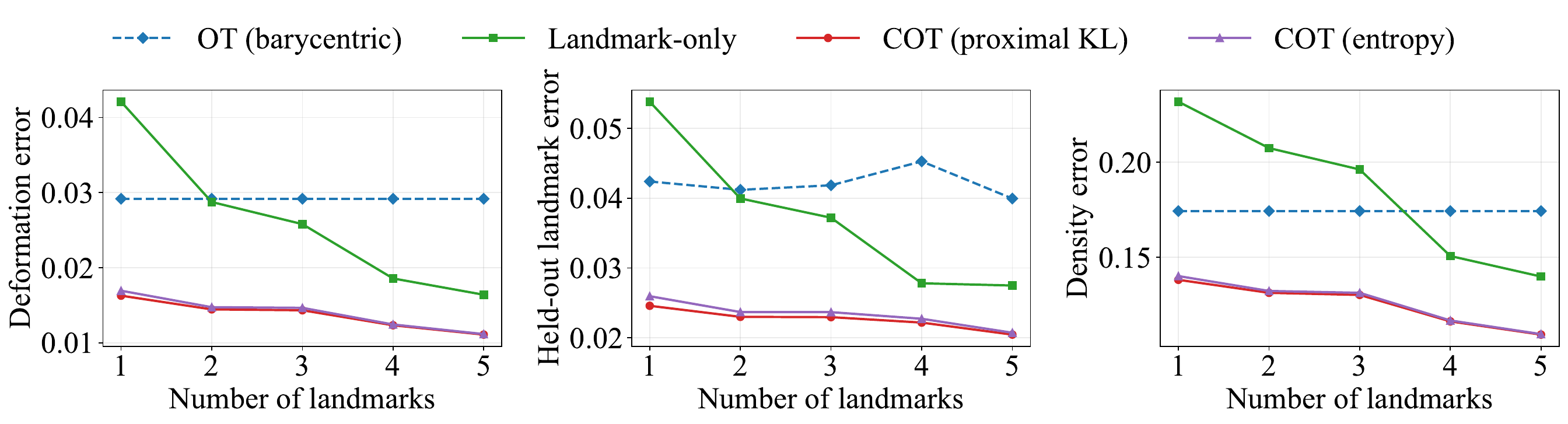}
    \caption{Deformation error (left), held-out landmark error (middle), and density error (right) of four methods under varying numbers of landmarks on the fish-shaped mask. Results are averaged over five replications.}
    \label{fig:metric}
\end{figure}
\paragraph{Performance of estimated deformation} We report the metric of the estimated deformation in Fig.~\ref{fig:metric}. We have the following observations from Fig.~\ref{fig:metric}. As the number of landmarks increases, both the Landmark-only using only landmarks and our proposed COT (proximal KL) and COT (entropy) consistently improve performance across all three metrics, confirming that the landmarks indeed help recover the deformation field under distribution transition. Although the method of OT (barycentric) directly matches the source and target distributions on the coarse computational grid, its density error remains higher than that of COT after extension to the fine grid. This may be attributed to the fact that OT (barycentric) estimates a transport plan rather than a spatially coherent deformation map. The deformation field is subsequently induced through barycentric projection, so its regularity and off-grid accuracy are not directly controlled by the OT objective. 
The benefit of OT prior is particularly evident in the sparse-landmark regime, where COT substantially outperforms Landmark-only, demonstrating that global distribution matching effectively complements the limited geometric information provided by sparse landmarks.
Meanwhile, as more landmarks become available, the gap between COT and Landmark-only  narrows, but COT remains consistently more accurate by combining global distribution matching with landmark constraints. Additionally, the proximal-KL and entropy-regularized variants perform nearly identically, indicating little sensitivity to the numerical scheme.

\begin{figure}
    \centering
    \includegraphics[width=1\linewidth]{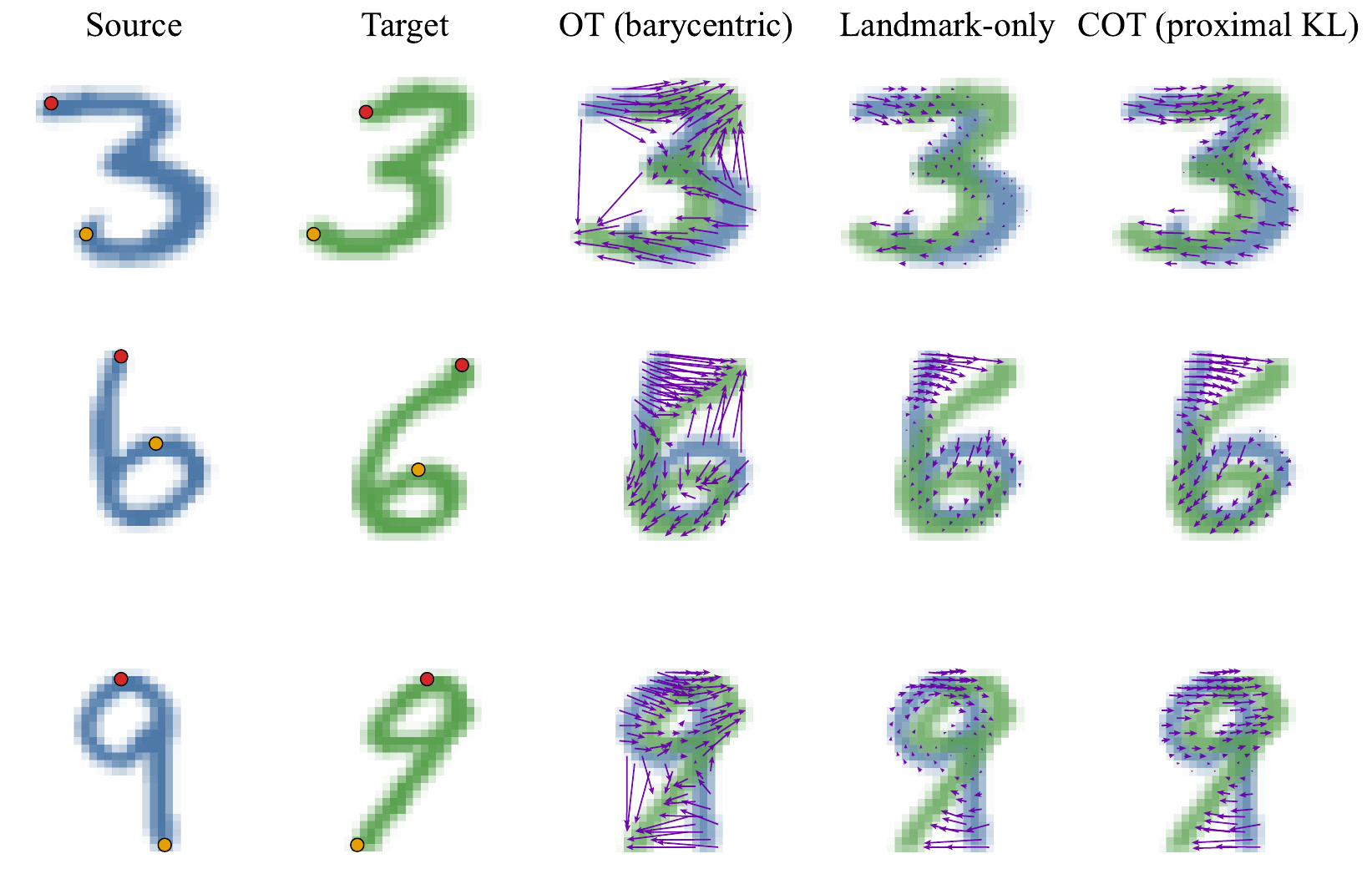}
    \caption{Comparison of the deformation fields estimated by different methods. In each row, the source and target shapes are shown in blue and green, respectively. Two landmark pairs indicated in red and orange are provided for each shape pair, and the estimated deformation fields are visualized by pink arrows. }
    \label{fig:mnist_shape_matching}
\end{figure}
\subsection{Real Shape Matching} We further evaluate the proposed method on real shape-matching problems. We use handwritten digit images from the MNIST dataset~\cite{lecun1998gradient} and select three representative source-target pairs, as shown in Fig.~\ref{fig:mnist_shape_matching}. For each pair, two landmarks are manually specified to provide sparse geometric guidance. Unlike the synthetic setting, the true deformation between two real shapes is unavailable, so the comparison is mainly based on the estimated deformation fields. Since COT (proximal KL) and COT (entropy) exhibit nearly identical performance in the synthetic experiments, we report only COT (proximal KL) here for clarity.

As shown in Fig.~\ref{fig:mnist_shape_matching}, OT (barycentric) captures the overall distributional correspondence but can produce spatially irregular displacement fields that are not fully consistent with the apparent shape structure. In contrast, Landmark-only produces smoother deformation fields around the prescribed landmarks, but the estimated deformation diminishes in regions where the sparse landmark constraints provide little guidance. Through the coupled estimation of transport and deformation, our proposed COT (proximal KL) propagates both distributional and landmark information across the source support, leading to more spatially coherent deformation fields that remain consistent with the prescribed correspondences. These results further demonstrate the benefit of coupling distributional alignment with sparse geometric constraints.

\section{Conclusions}
\label{sec:conclusions}
This paper proposed a novel coupled optimal transport (OT) framework that jointly optimizes the transport plan and deformation field, integrating distributional alignment induced by transport-cost minimization and geometric information provided by sparse landmarks. We established the well-posedness of the proposed model and characterized its connection to classical OT. Finite-element-based numerical algorithms were further developed, together with their convergence analysis. Synthetic and real shape-matching experiments demonstrated the effectiveness of the proposed coupled OT. 

While the proposed framework is general, the current finite-element-based numerical algorithms were primarily designed for low-dimensional spatial domains. Developing numerical methods for coupled OT for high-dimensional and large-scale problems is an important direction for future research. In particular, parameterizing the deformation field with neural networks may provide a flexible and scalable alternative to finite-element representations, potentially enabling coupled transport-deformation modeling for more complex data and applications.

\appendix
\section{Proof of Remark \ref{thm:prop_alpha_1}} 
\begin{proof}
    For convenience, we denote 
    $$R(\pi,u)=:\int h(T_{u}(x)-y) \dif\pi(x,y), \mbox{ and } C(\pi)=:\int c(x,y)\,\dif\pi(x,y).$$
    Based on Theorem~\ref{thm:exsitence}, we have 
    $ J(u_n)\leqslant C, \mbox{ and }R(\pi_n,u_n)\leqslant C,$
    for some $C>0$. According to the proof of Theorem~\ref{thm:exsitence}, up to a subsequence, $u_{\alpha_n}\to u^*$ strongly in $L^2(\mu;\mathbb R^p)$, and $\pi_{\alpha_n}\rightharpoonup \pi_\star $ narrowly in $\mathcal P(\Omega\times\Omega)$, where $(\pi^*,u^*)\in\Pi(\mu,\nu)\times \mathcal{U}$.
    It remains to identify the limit. By lower semicontinuity,
$$
\begin{aligned}
R(\pi^*,u^*) + J(u^*) \leqslant
\liminf_{n\to\infty}
R(\pi_n,u_n)
+
J(u_n).
\end{aligned}
$$
On the other hand, for any \((\pi,u)\in\Pi(\mu,\nu)\times\mathcal U\), the
minimality of \((\pi_n,u_n)\) gives
$$
\begin{aligned}
(1-\alpha_n)
c(\pi_n)
+
\alpha_n
R(\pi_n,u_n)
+
J(u_n)
\leqslant
(1-\alpha_n)
C(\pi)
+
\alpha_n
R(\pi,u)
+
J(u).
\end{aligned}
$$
Dropping the first nonnegative term on the left-hand side, we have
\[
\begin{aligned}
R(\pi_n,u_n)+J(u_n)
&=
\alpha_n R(\pi_n,u_n)+J(u_n)
+(1-\alpha_n)R(\pi_n,u_n) \\
&\leqslant
(1-\alpha_n)C(\pi)+\alpha_n R(\pi,u)+J(u)
+(1-\alpha_n)R(\pi_n,u_n).
\end{aligned}
\]
Since \(R(\pi_n,u_n)\) is uniformly bounded and \(\alpha_n\to 1\), we have
\[
\limsup_{n\to\infty}
\left[
R(\pi_n,u_n)+J(u_n)
\right]
\leqslant
R(\pi,u)+J(u).
\]
Therefore,
$$
R(\pi^*,u^*)
+
J(u^*)
\leqslant
R(\pi,u)
+
J(u)
$$
for every \((\pi,u)\in\Pi(\mu,\nu)\times\mathcal U\). Hence, we conclude the proof.
\end{proof}

\section*{Acknowledgments}
We would like to thank the referees for handling our paper.

\bibliographystyle{siamplain}
\bibliography{references_SIAP_final}
\end{document}

%% file: ex_shared.tex
\usepackage{lipsum}
\usepackage{amsfonts}
\usepackage{graphicx}
\usepackage{epstopdf}
\usepackage{algorithmic}
\ifpdf
  \DeclareGraphicsExtensions{.eps,.pdf,.png,.jpg}
\else
  \DeclareGraphicsExtensions{.eps}
\fi

\newsiamremark{remark}{Remark}
\newsiamremark{hypothesis}{Hypothesis}
\crefname{hypothesis}{Hypothesis}{Hypotheses}
\newsiamthm{claim}{Claim}
\newsiamremark{fact}{Fact}
\crefname{fact}{Fact}{Facts}

\headers{Coupled Optimal Transport}{X. Gu, J. Sun, and Z. Xu}

\title{Coupled Optimal Transport with \\ Landmark Constraints}

\author{Xiang Gu\thanks{School of Mathematics and Statistics, Xi'an Jiaotong University, China
  (\email{xianggu@xjtu.edu.cn, jiansun@xjtu.edu.cn, zbxu@xjtu.edu.cn}).}
\and Jian Sun  \footnotemark[1]
\and Zongben Xu\footnotemark[1]
}

\usepackage{amsopn}
